\pdfoutput=1

\documentclass[11pt]{article}

\usepackage[table]{xcolor}

\usepackage{acl}

\usepackage{times}
\usepackage{latexsym}
\usepackage{float}

\usepackage{pgfplots}

\usepackage{multirow}
\usepackage{multicol}
\usepackage{booktabs}
\usepackage{siunitx}
\usepackage[T1]{fontenc}

\usepackage[utf8]{inputenc}

\usepackage{microtype}

\usepackage{inconsolata}
\usepackage{graphicx}
\usepackage{algorithmicx}
\usepackage{algpseudocode}
\usepackage{algorithm}
\usepackage{subcaption}
\usepackage{tikz}
\usepackage{hyperref}
\usepackage{tcolorbox}
\usepackage{textcomp}

\usepackage{array}
\usepackage{amsthm}
\usepackage{amssymb}
\usepackage{amsmath}
\usepackage{bbold}

\theoremstyle{plain}
\newtheorem{prop}{Proposition}
\newtheorem{thm}{Theorem}

\theoremstyle{definition}
\newtheorem{defn}{Definition}

\theoremstyle{remark}
\newtheorem{rem}{Remark}

\newcommand{\Esp}{\mathbb{E}}

\newcommand{\1}[0]{\mathbb{1}}

\newcommand{\Vrond}{\mathcal V}

\newcommand{\vertiii}[1]{{\left\vert\kern-0.25ex\left\vert\kern-0.25ex\left\vert #1 
    \right\vert\kern-0.25ex\right\vert\kern-0.25ex\right\vert}}

\newcommand{\set}[1]{\left\{ #1\right\}}

\renewcommand{\geq}{\geqslant}
\renewcommand{\leq}{\leqslant}

\newcommand{\Ind}{\mathbb{1}}

\newcommand{\follows}{\sim}

\def\RD{R_{C,\ell}}
\def\RDH{R_{C,\ell_{01}}}

\title{\texttt{COSMIC}: Mutual Information for Task-Agnostic Summarization Evaluation}

\author{ {\bf Maxime \textsc{Darrin}$^{1, 2, 3, 4}$} \quad {\bf Philippe \textsc{Formont}$^{1, 2, 4, 5}$} \\ {\bf Jackie Chi Kit \textsc{Cheung}$^{2, 3, 7}$} \quad {\bf Pablo \textsc{Piantanida}$^{1, 2, 4, 6}$} \\  $^{1}$International Laboratory on Learning Systems, $^{2}$Mila - Quebec AI Institute \\ $^{3}$McGill University $^{4}$Université Paris-Saclay, $^{5}$École de technologie supérieure (ETS) \\ $^{6}$CNRS, CentraleSupélec {$^7$Canada CIFAR AI Chair} \\ { \url{maxime.darrin@mila.quebec} \quad \url{philippe.formont@mila.quebec} } \\ { \url{jackie.cheung@mcgill.ca}\quad \url{pablo.piantanida@mila.quebec} } }

\algrenewcommand\algorithmicrequire{\textbf{Input:}}
\algrenewcommand\algorithmicensure{\textbf{Output:}}

\begin{document}
    \maketitle

    \begin{abstract}
    Assessing the quality of summarizers poses significant challenges---gold summaries are hard to obtain and their suitability depends on the use context of the summarization system. Who is the user of the system, and what do they intend to do with the summary? In response, we propose a novel task-oriented evaluation approach that assesses summarizers based on their capacity to produce summaries while preserving task outcomes. We theoretically establish both a lower and upper bound on the expected error rate of these tasks, which depends on the mutual information between source texts and generated summaries. We introduce \texttt{COSMIC}, a practical implementation of this metric, and demonstrate its strong correlation with human judgment-based metrics, as well as its effectiveness in predicting downstream task performance. Comparative analyses against established metrics like \texttt{BERTScore} and \texttt{ROUGE} highlight the competitive performance of \texttt{COSMIC}.
    
    \end{abstract}

    \section{Introduction}

    Assessing the quality of summarizers in different settings, tasks, and datasets is critical for better understanding these models and for studying their strengths and weaknesses. In many text generation scenarios, assessing model quality is arduous and resource-intensive, often necessitating human annotations and evaluations. Consequently, developing automatic metrics that align closely with human judgments is paramount~\cite{graham-baldwin-2014-testing, tratz2008summarization, 10.1145/1410358.1410359, deutsch2021towards}.

Standard automatic evaluation methods for summarization rely on the idea that a good summary should have some (semantic) overlap with either a gold standard or the source text~\cite{el2021automatic, allahyari2017text}. They leverage similarity metrics such as BLEU~\cite{papineni2002bleu} or ROUGE~\cite{lin-2004-rouge} to evaluate the quality of the generated summaries. However, these methods often do not correlate well with human judgements~\cite{kryscinski-etal-2020-evaluating, kocmi-etal-2021-ship}. To enhance alignment with human judgment and capitalize on recent advancements in large language models, recent efforts have concentrated on learned metrics for scoring summaries~\cite{zhang2020bertscore, rei-etal-2020-comet, liu2023geval}. For instance, \citet{clark2023seahorse} introduced six new learned metrics by finetuning a pretrained MT5 model~\cite{xue2021mt5} to predict human judgment along different axes. Another line of research focuses on reconstruction-based metrics, such as BLANC~\cite{vasilyev-etal-2020-fill} and Information Difference~\cite{egan2021play}, the former evaluating the actual reconstruction error of the source text with and without a summary, and the latter evaluating the information gain obtained when conditioning the generative model on the summary.

Standard evaluation methods, therefore, suffer from two major shortcomings. They rarely formally define a clear notion of quality and the methods used to evaluate said notion lack theoretical foundations. This leads to potential discrepancies between the intended evaluation and what is actually measured, leading to a lack of validity of the evaluation methods.

 We introduce a task-oriented evaluation setup where summaries are meant to allow an agent to perform downstream tasks without reading the longer source text. For example, if a political advisor drafts a briefing on a subject to enable a politician to make informed decisions, our evaluation considers the advisor successful if the decisions made using the summary align with those made using the initial source text~\cite{pu2023summary, VANDERWENDE20071606}. Providing such a well-defined notion of success enables careful analysis of the validity of the evaluation for different possible use cases. In addition, it enables us to perform mechanical evaluations of our method, reducing the noise from the evaluation process.

 Furthermore, many summarization techniques typically operate under the assumption that the downstream task is unknown during the summarization process, yet an implicit notion of summary quality exists. However, this assumption is seldom explicitly addressed or substantiated. In our task-oriented approach, we provide an information-theoretic rationale for the existence of such a metric. We demonstrate that it essentially involves assessing the mutual information (MI) between the source texts' distribution and the summaries generated by a given summarizer.

This paper does not aim to study the syntactic, semantic, and pragmatic aspects of features of summary information, which are essential for capturing the rich notion of information in human communication. 
For instance, we do not account for the summaries' fluency, grammatical correctness, or coherence. Instead, our approach emphasizes how effectively, in the best-case scenario, the output of a summarizer can be used to perform downstream tasks. The results in this paper suggest that the formal definition of mutual information successfully achieves this goal.  Our approach leverages embeddings to abstract the surface form of text, following recent studies by \citet{pillutla2021mauve} and \citet{pimentel2023usefulness}.

We evaluate the quality of our approach in two ways. First, we show that summarizers that induce a summary distribution with higher MI with the source texts' distribution are higher quality in the following sense --- \emph{they tend to produce summaries that preserve outcomes on downstream tasks} as compared to using the source texts. Second, we compare the MI to metrics trained on human judgments and show that it displays consistent correlations. Our results are consistent with and extend previous work that leverages MI to understand relationships in datasets~\cite{ethayarajh2022understanding, bugliarello2020its}, predictive performance of representations~\cite{sui2023evaluating} and tool to construct or evaluate  representations~\cite{kim2022mutual}, models and generative processes.\vspace{1mm}

    \noindent\textbf{Contributions.} Our contributions are threefold:
    \begin{enumerate}
        \item \textbf{A theoretical setting for summarizer evaluation.} We frame the summarizer evaluation problem as a statistical inference problem and we derive a task-agnostic reference-free quality metric: the MI between the distribution of source texts and the distribution of the summarizer's outputs.
        \item \textbf{A practical implementation of this metric: \texttt{COSMIC}.} We propose a practical implementation of \texttt{COSMIC} using an MI estimator and sentence embeddings~\footnote{A plug\&play python library built on top of HF transformers is available as supplementary material and will be released upon publication of this work.}.
        \item \textbf{An experimental evaluation.}
        We examine how well MI predicts the performance of downstream tasks in comparison to conventional metrics like \texttt{BERTScore} and \texttt{BARTScore}. Our findings demonstrate that MI is competitive with these metrics. Additionally, we illustrate its strong correlation with metrics trained to emulate human judgment.
    \end{enumerate}

    \section{Related Work}

   Assessing the quality of summaries poses a unique challenge due to the contextual nature of `quality', influenced by factors such as the audience, topic, and intended purpose of the summary. Even expertly crafted human summaries considered "gold standard" in one context may be perceived as subpar in a different setting~\cite{saziyabegum2017review, indu2016review}.\vspace{1mm}

    \noindent\textbf{Reference-free summary evaluation.}  Reference-free evaluation methods mostly rely on comparing the content of the summary with the content of the source text~\cite{louis-nenkova-2013-automatically, el2021automatic} and they rely on common overlap metrics such as ROUGE~\cite{lin-2004-rouge}, BLEU~\cite{papineni2002bleu} or BERTScore~\cite{zhang2020bertscore}. However, most reference-free metrics show some important limitations~\cite{deutsch-etal-2022-limitations}: they lack theoretical grounding.  Work such as BLANC~\cite{vasilyev-etal-2020-fill} and Information Difference~\cite{egan2021play} are closely related to our work. The former evaluates the actual reconstruction error of the source text with and without a summary, while the latter evaluates the information difference evaluated by the generative model. However, both lack the theoretical justification we introduce in \autoref{sec:task_driven_framework} and only evaluate the information based on the generative model. While these methods appear intuitive, they lack of theoretical support for their approaches, relying solely on empirical results and correlations with human evaluations. Thus, the results are often tied to a dataset and are affected by variance from the human evaluation. Conversely, we offer a theoretical framework applicable to any dataset and measure a tangible, well-defined success criterion: do the tasks yield the same output when performed on summaries as they do on the source texts?
\vspace{1mm}

    \noindent\textbf{Embedding-based evaluation.} \texttt{MAUVE}~\cite{pillutla2021mauve} first proposed a new information-theoretic metric to compare two text distributions based on embedding clustering; more recent work showed that the crucial element was the clustering step~\cite{pimentel2023usefulness}. They show that while embeddings (and the clusters they form) do not capture fluency or grammatical correction, they do grasp meaning and coherence, making them excellent tools for evaluation. %
\vspace{1mm}

    \noindent\textbf{Dataset difficulty and MI}. Measuring MI between the concepts and input in a dataset is not a novel idea. In fact, \citet{ethayarajh2022understanding} leverages the Arimoto information~\cite{arimoto1971information}, rediscovered and dubbed $\Vrond$-usable information by~\citet{xu2020theory} to assess the difficulty of a dataset. Similarly, \citet{bugliarello2020its} evaluate the difficulty of translating from one language to another. Following this trend,~\citet{kim2022mutual} reuses the point-wise MI between Gaussian distributions to evaluate text-to-images and image-to-text generative models.\vspace{1mm}

    \noindent\textbf{Mutual information for summarization.} The MI is a natural metric to optimize in summarization. It has been used as a score to select the most informative or surprising sentences in extractive summarization~\cite{padmakumar2021unsupervised} or as an alternative objective for text decoding~\cite{vanderpoel2022mutual}. In this paper,  we revisit the use of  MI but between the distribution of source texts and the distribution of the summarizer.

    \section{A Task-Driven Evaluation Framework}
    \label{sec:task_driven_framework}

\subsection{Background: Probabilistic models for text summarization}
We consider models for language summarization tasks that define a probability distribution over strings. More formally, these models are probability distributions $s$ over an output space $\mathcal{S}$ conditioned on an input text $\mathbf{t}$, where $\mathcal{S}$ is the set consisting of all possible strings that can be constructed from the vocabulary $\Omega$:
$ 
\mathcal{S} \triangleq \big\{ \texttt{BOS}  \circ \mathbf{s}  \circ \texttt{EOS}\,  | \,  \mathbf{s} \in \Omega^* \big\}, 
$
$\texttt{BOS}$ and $\texttt{EOS} $ stand for special reserved beginning-of-sequence and end-of-sequence tokens, respectively, and $\Omega^* $ denotes the Kleene closure of $\mathcal{S}$.

Today’s models for language summarization are typically parameterized by encoder-decoder or decoder-only architectures with attention mechanisms with trainable weights $\theta$. These models follow a local-normalization scheme, meaning that $\forall$ $i > 0$, $ p_\theta (\cdot | \mathbf{s}_{<i}, \mathbf{t} )$ defines a probability distribution over $\bar{\mathcal{S}} = \mathcal{S} \cup \texttt{EOS}$. The probability of a sequence $\mathbf{s}  = \langle s_0, s_1, \ldots \rangle$ can then be decomposed as:
\begin{equation}
    p_\theta (\mathbf{s}|\mathbf{t}) = \prod_{i=1}^{|\mathcal{S}|} 
   p_\theta (\mathbf{s}_i | \mathbf{s}_{<i},\mathbf{t}),
    \label{eq:sequence_probability}
\end{equation}
and $\mathbf{s}_{<i} = \langle s_0, \ldots, s_{i-1} \rangle$,  $s<1 = s_0 \triangleq  \texttt{BOS}$.

\subsection{Background: Information Theory} 

Information theory \cite{cover91} provides several tools for analyzing data and their associated probability distributions, including entropy and MI. These metrics are typically defined based on a "true" probability distribution, denoted as $p(c)$,  or the joint probability density function $p(\mathbf{t}, \mathbf{s})$ which may not be known but governs the behavior of random variables $\mathbf{T}$ and $\mathbf{S}$. The fundamental concept in information theory is \textbf{surprisal}, defined as $H(C= c) = -\log p(c)$,  and its expected value is termed \textbf{entropy:}
\begin{equation}
H(C) = \sum_{c \in \mathcal{C}} p(c) H(C = c). 
\end{equation}
Finally, another important concept is the \textbf{mutual information} (MI) between two random variables:
\begin{equation}
I(\mathbf{T};\mathbf{S}) = H(\mathbf{T}) - H(\mathbf{T}|\mathbf{S}). 
\end{equation} It captures the amount of information we get about one random variable when observing a realization of the other.
\textbf{The data-processing inequality} \cite{cover91} states that processing a random variable with a (possibly random) function $f(\cdot)$ can never increase its informativeness but only reduce its information content, expressed as: 
\begin{equation}
I(C;f(\mathbf{T})) \leq I(C;\mathbf{T}).\label{eq-data-processing}
\end{equation}
    The \textbf{rate-distortion} (RD) function of a discrete random variable $C$ for a given distortion function $\ell(c, \widehat{c})$ is defined as \cite[eq.~(1.4)]{csiszar74}:
    \begin{equation}
        R_{C,\ell}(D) \,\triangleq \min_{\rule{0mm}{4.3mm}\substack{p(\widehat{c}|c)}\, :\\ \, \Esp[\ell(C,\widehat{C})] \leq D} \!\! I(C;\widehat{C}).    
    \end{equation}
For further details, the reader is referred to Appendix~\ref{sec:RD-function} and \cite{cover91}.

\begin{figure}
    \centering
    \includegraphics[width=0.5\textwidth]{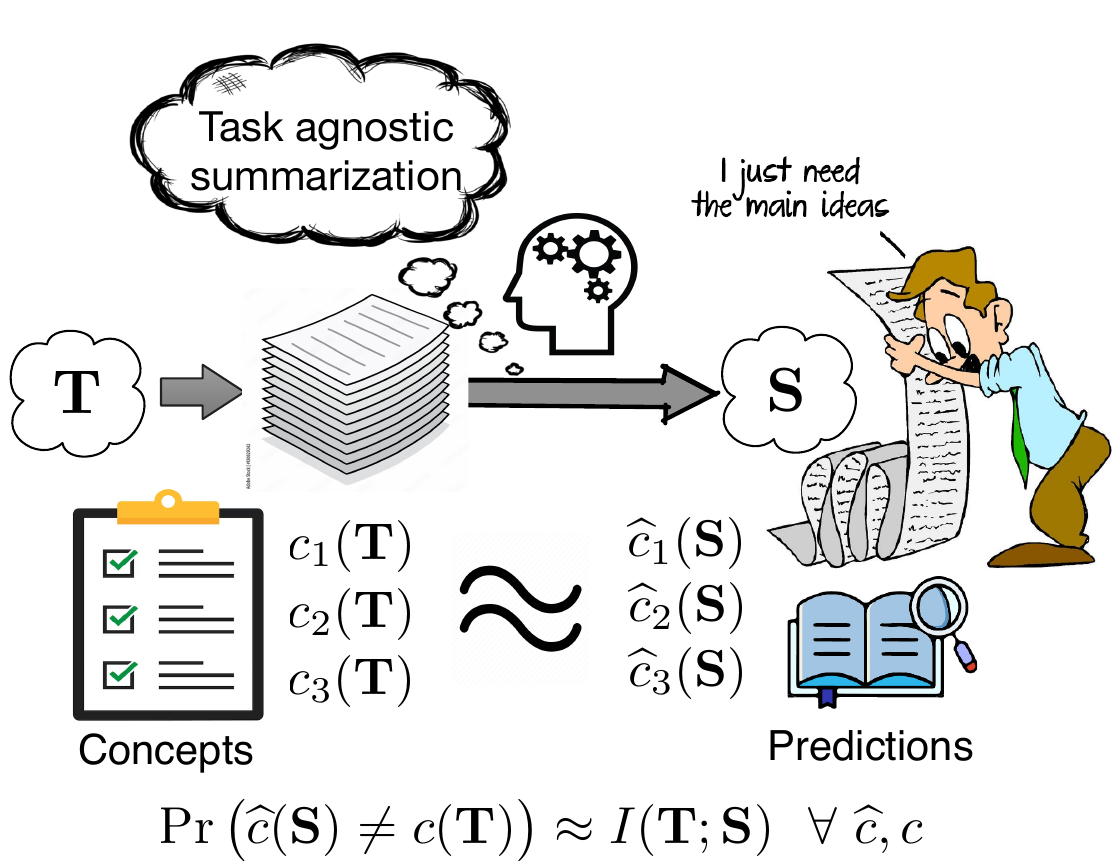}
    \caption{The summarizer is expected to generate summaries $\mathbf{S}$ for a given distribution of source texts $\mathbf{T}$, without prior knowledge of the specific application, e.g. predicting the concepts: $c_1,c_2,\dots $. The objective is to assess the discrepancy incurred when predicting from the generated summaries instead of the original source texts. We demonstrate that, for any undisclosed task $c$, the likelihood of error is constrained within bounds determined by  monotonous functions of the MI.}
    \label{fig:schema}
\end{figure}

\subsection{A task-oriented evaluation setting} 

Most summarization methods operate under the assumption that the downstream task of interest is unknown during the generation process, relying instead on a generic notion of summary quality. We assume in this work that the goal is to be able to perform similarly in terms of classification error on both the original texts $\mathbf{T}$ and the resulting summaries $\mathbf{S}$\footnote{Other goals are possible such as reducing the source text complexity, removing specific information etc... Our choice stems from the practical usability of the task-oriented definition for evaluation purposes}. Next, we formalize this evaluation metric based on the assumption of an unknown downstream task. 

Let $c : \Omega^* \rightarrow  \{1, \dots, m\}$ denote the target concept of interest which can be extracted from the initial texts $\mathbf{T}$ by applying $C\triangleq c(\mathbf{T})$; and similarly, let $\widehat{c}(\mathbf{S})  $ denote the predicted concept from the summaries $\mathbf{S}$ according to the underlying model $p_\theta (\mathbf{s}|\mathbf{t})$. The evaluation of the summarizer's quality with respect to the downstream task (unknown from the summarization model) can be assessed using the expected error rate, as illustrated in \autoref{fig:schema}. In other words, it involves the classifier,  determining the average probability of the extracted concept $\widehat{c}(\mathbf{S})  $ differing from the original concept $c(\mathbf{T})$ in the source text: 
\begin{equation}
P_e (c,\widehat{c},\theta) \triangleq  \Esp^{(\theta)} \big[  \1[ c(\mathbf{T}) \neq  \widehat{c}(\mathbf{S})] \big],\label{eq-perf-error}
\end{equation}
where the expectation is taken with respect to the joint distribution of text and summary $(\mathbf{T},\mathbf{S})$ based on the source texts distribution and the distribution over the summaries induced by the summarizer. Interestingly, it is not difficult to check by data-processing inequality that \eqref{eq-data-processing}: 
$
I(C;\mathbf{S})  \leq   I(C;\mathbf{T})
$
and thus, the output summaries $\mathbf{S}$ by the summarization model may not preserve all relevant information necessary to predict $C$ from $\mathbf{S}$, unless the summarization model is aware of the downstream task. However, identifying the set of relevant downstream tasks for various texts is challenging in practice. Consequently, \eqref{eq-perf-error} does not provide a satisfactory evaluation metric for summarization systems. This observation raises the central question studied in this work: 
\textbf{How much information about any downstream task is preserved in the summaries?}     

\section{A Task-Agnostic Quality Measure for Summarization}

\subsection{Theoretical Results}
\label{sec:theoretical_results}

In this section, we rigorously motivate the evaluation of summarization systems by measuring the MI between texts and the resulting summaries. Let's assume the existence of a random variable $\mathbf{T}$ representing source texts, which follows an unknown distribution inherent to the source text-domain.  Given a stochastic summarization system $p_\theta (\mathbf{s}|\mathbf{t})$ as defined in expression  \eqref{eq:sequence_probability},  we denote $\mathbf{S} \follows p_\theta (\mathbf{s}|\mathbf{T})$ the random variable representing summaries generated by the summarization system for these sources texts $\mathbf{T}$.

Consider the task $c(\mathbf{T})$ that we intend to execute on the source texts, where $C = c(\mathbf{T})$ represents the random variable denoting the outcomes of this task on source texts. We assume that the same task can be performed on the summaries and let  $\widehat{c}(\mathbf{S})$ denote the corresponding prediction.   Intuitively, if we can accurately predict the outcome of $C$ from the summaries, then we will say that the summarization system is high-quality since it preserves the necessary information for performing the task. The next proposition frames our information-theoretic bounds on the performance of any arbitrary downstream task. In particular, it shows that the expected error rate $P_e (c,\widehat{c},\theta) $ as defined in \eqref{eq-perf-error} can be upper and lower bounded by the MI between the texts and the summaries.

     \begin{prop}[Information-theoretic bounds]    \label{prop:prop_bounds} 
    Let $C = c(\mathbf{T}) $ denote the underlying concept variable and let $ \widehat{c}(\mathbf{S}) $ be the Bayes predictor of $C$ observing the output summaries $\mathbf{S}$, based on the  underlying summarization model $p_\theta (\mathbf{s}|\mathbf{t})$. The expected error rate satisfies: 
        \begin{align}
           P_e (c,\widehat{c},\theta) & \leq  1 -  \kappa  \exp\big(  I(\mathbf{T} ; \mathbf{S}) \big), \label{eq-upper-b}\\
            P_e (c,\widehat{c},\theta) & \geq \RDH^{-1}\big(I(\mathbf{T};\mathbf{S}) \big), \label{eq-upper-l} 
        \end{align}
       where  $\kappa \in (0,1)$ is a constant which does not depend on summaries; and $\RDH^{-1}(\cdot)$ is the inverse of the rate-distortion function using  $\ell_{01}(c,\widehat c) = \Ind[c \neq  \hat c]$. Furthermore, the lower bound holds for an arbitrary loss $\ell(\cdot,\cdot)$ measuring the disagreement between the concept and its predictive value: 
        \begin{equation*}
        \inf_{\widehat{c}(\cdot)} \Esp[\ell(C, \widehat{c}(\mathbf{S}))] \geq     \RD^{-1}\big(I(\mathbf{T};\mathbf{S}) \big).
        \end{equation*}
        \end{prop}
\begin{proof}
The upper bound \eqref{eq-upper-b} relies on the fact that the predictor $\widehat{c}(\mathbf{S}) $ is the optimal (Bayes) classifier for which the expected error rate admits a well-known expression. The lower bound \eqref{eq-upper-l} uses data-processing inequality which implies that $I(\mathbf{T} ; \mathbf{S}) \geq I\big(c(\mathbf{T}) ; \widehat{c}(\mathbf{S})\big) $ and the definition of the rate-distortion function evaluated in the loss $\ell_{01}(c,\widehat c)$ noticing that its expectation yields the expected error rate   $P_e (c,\widehat{c},\theta)$ as defined in \eqref{eq-perf-error}. For further details, see Appendix~\ref{sec:full_proof}. 
\end{proof}

\begin{rem}
 The bounds in Proposition~\ref{prop:prop_bounds} imply that the expected error rate of any task predicted by observing the summaries is lower- and upper-bounded by functions of the MI between the text and the summaries $I(\mathbf{T} ; \mathbf{S})$. More precisely, the upper bound \eqref{eq-upper-b} is a monotonically decreasing function of the MI  while the lower bound \eqref{eq-upper-l} is a non-increasing function in the MI. The lower bound can be further simplified by evaluating the rate-distortion function, as shown in Appendix~\ref{sec:ap_rate_distortion}.  In other words, \textbf{greater MI  corresponds to improved the expected prediction performance on the summaries. Conversely, the expected error rate is bounded from above when MI is limited.} Interestingly, the arguments regarding the bounds do not depend on the considered task, suggesting that the MI can be used as a task-agnostic metric to evaluate the quality of summaries.

\end{rem}
  In the next section, we propose a practical method to estimate MI. In \autoref{sec:experimental_results}, we empirically show that it correlates well with the performance of the downstream tasks with other human judgment-based metrics and compare it to standard metrics. %

    \subsection{Estimating MI from samples}
    \label{sec:mi_estimation}

    While the MI captures an intuitive notion of information and is theoretically grounded,  estimating it accurately is notoriously challenging.  Therefore, following prior research~\cite{ethayarajh2022understanding, xu2020theory} we estimate Arimoto information~\cite{arimoto1971information} based on the KNIFE estimator~\cite{pichler2022differential}. Our method comprises three steps: (1) we project the source texts and the summaries into a continuous embedding space; (2) we fit a mutual information estimator onto these embeddings; and (3) we estimate the mutual information between the source texts and the summaries using the fitted estimator. We report the details of our method in Algorithm~\ref{alg:mi_estimation}.\vspace{1mm}

    \noindent\textbf{Mutual information estimator.} We rely on the KNIFE estimator~\cite{pichler2022differential} to estimate the differential entropy of the embeddings and then  mutual information. This estimator effectively relies on Gaussian Mixtures with $K$ modes to fit the density function and induces a soft-clustering for text generation evaluation~\cite{pillutla2021mauve, pimentel2023usefulness}. We found experimentally that the number of modes $K$ did not impact the performance significantly. We report the results for $K=4$ (see Appendix \ref{sec:ap_mi_estimation} for further details).\vspace{1mm}

    \noindent\textbf{Embeddings.} In order to obtain continuous representations of the source texts and the summaries, we mainly rely on the AnglE-embedders~\cite{li2023angleoptimized} as they are at the top of the MTEB leaderboard~\cite{muennighoff2023mteb} with a rather small model. In addition, we experimented with different embedders from the sentence transformers library~\cite{reimers-2019-sentence-bert}, mainly the paraphrase and sentence similarity embeddings. We find that our method was robust in the choice of the embedder and thus reported the results for the AnglE-embedders.

    \begin{algorithm}
        \caption{Evaluating the performance of a summarizer using KNIFE and sentence transformers}
        \label{alg:mi_estimation}
        \begin{algorithmic}[1]
            \Require{A dataset $\mathcal{D}_N = \set{(\mathbf{T}_i, \mathbf{S}_i)}_{i=1}^N$ }
            \Require{A pre-trained embedder $\operatorname{Emb}$}
            \Ensure{An estimation of the mutual information between $\mathbf{T}$ and $\mathbf{S}$}
            \State{$ \mathbf{E_T} \gets \set{\operatorname{Emb} (\textbf{T}_i)}_{i=1}^N$}
            \State{$ \mathbf{E_S} \gets \set{\operatorname{Emb} (\textbf{S}_i)}_{i=1}^N$}

            \State{$\widehat{I}_N(\textbf{T} ;  \textbf{S}) \gets \operatorname{KNIFE}(\mathbf{E_T}, \mathbf{E_S})$} \\
            \Return{$\widehat{I}_N(\textbf{T} ;  \textbf{S}) $} 
        \end{algorithmic}
    \end{algorithm}

    \section{Experimental Settings}
    \label{sec:experimental_settings}

While the bounds derived in \autoref{sec:theoretical_results} provide certain theoretical guarantees regarding the mutual information, their practical consequences must still be evaluated empirically. We present here two evaluation methods. First, we show that MI effectively predicts whether performing a downstream on the summaries would lead to the same outcome as if it were performed on the source text. This validates our task-oriented evaluation approach for summaries. Furthermore, we contrast our metric with metrics trained to emulate human judgments across various dimensions. Remarkably, we observe that even without any training, our metric closely aligns with human preferences.\vspace{1mm}
 
    \noindent\textbf{Datasets.} We select three summarization datasets for the English language: CNN/DailyMail~\cite{see-etal-2017-get, DBLP:conf/nips/HermannKGEKSB15}, XSum~\cite{Narayan2018DontGM} and MultiNews~\cite{fabbri-etal-2019-multi} and perform all evaluations on all datasets. We provide additional experiments in French and Spanish using the MLSUM~\cite{scialom2020mlsum} and XLSUM~\cite{hasan-etal-2021-xl} datasets. We report mixed results due to the lack of efficient multilingual embedders in \autoref{sec:additional_languages}.
\vspace{1mm}

    \noindent\textbf{Models.} We evaluate numerous summarizers from the HuggingFace hub, relying on different backbones, pretraining methods and finetuned on different datasets. We conduct our experiments on the PEGASUS suite of models~\cite{zhang2020pegasus}, on BERT large models~\cite{DBLP:journals/corr/abs-1910-13461} and on DistilBERT models~\cite{shleifer2020pretrained}. We also evaluated the models presented in the SEAHORSE benchmark~\cite{clark2023seahorse} through the generated summaries they proposed, which include MT5 models~\cite{xue2021mt5}, different variants of T5 models~\cite{raffel2023exploring}, and of the PaLM models~\cite{chowdhery2022palm}. The reader is referred to \autoref{sec:models_app} for models' details.\vspace{1mm}

    \noindent\textbf{Baseline metrics for quality assessment.} For all summaries, we evaluate the quality estimation metrics obtained by computing reference-free versions of the \texttt{ROUGE-L}, \texttt{BERTScore}~\cite{zhang2020bertscore} and \texttt{BARTScore}~\cite{NEURIPS2021_e4d2b6e6} between the source texts and the summaries. We average over the dataset to get a summarizer-level score that we can compare with the MI.

    \noindent\textbf{Human evaluation with SummEval.}  While the SummEval~\cite{fabbri2021summeval} dataset provides only a very limited number of human judgments, it contains enough aligned source, and AI-generated summaries to evaluate the MI metric. We use these unannotated data to evaluate the MI and rank the model accordingly; then we compare this ranking with the human evaluation provided by the SummEval benchmark. We provide a comparison with previous metrics in \autoref{tab:additional_baselines_human_eval}. It is worth noting that the evaluation is not fair for our metric as other are evaluated only on the annotated data, while ours is evaluated on the unannotated data.

    \subsection{Downstream tasks}

To demonstrate the efficacy of our metric in predicting the downstream task performance on the summaries, the ideal scenario would involve generating summaries and tasking various individuals with different tasks across diverse contexts. However, executing this at a larger scale is impractical. As an alternative, we suggest comparing the results of different algorithms (classifiers and embedders) applied to both the source texts and the generated summaries. A summarizer can be deemed effective if these algorithms yield consistent results when applied to both the summaries and the source texts.\vspace{1mm}

    \noindent\textbf{Classification tasks.} We selected four different tasks (sentiment analysis, policy classification, Emotion classification and ChatGPT detector) and corresponding classifiers from the Huggingface Hub to run on the source texts and summaries and compared their outputs. We report the expected error rate\eqref{eq-perf-error}, \textit{i.e.}, the classifier outputting a different label on the source text and the summary.

    \noindent\textbf{Embedders.} We compare the embeddings obtained from the source texts and the summaries from different models, the output of the embedding by the classifiers models, and paraphrase and sentence similarity embeddings~\cite{reimers-2019-sentence-bert}. We show the correlation of our metrics with the cosine similarity between the embeddings of the source texts and the corresponding summaries. Since embedders are supposed to abstract the information in the texts, good summarization models should produce embeddings close to the source texts' embeddings. 

    \subsection{Correlation With Learnt Metrics}

    There are not many available metrics trained on human judgment for summarization. We chose to evaluate our metric on the SEAHORSE metrics~\cite{clark2023seahorse} as they are the most recent and provide interpretable metrics.\vspace{1mm}

    \noindent\textbf{Seahorse metrics.}  These metrics --- learnt from human judgement --- assess summaries along $6$ axes: \textbf{Main ideas; Attribution; Comprehensible; Grammar; Repetition and Conciseness}. \texttt{Main ideas} and \texttt{Attribution}, \texttt{respectively}, measure if the main ideas of the source text are present in the summary and if the information in the summary does indeed come from the source text. \texttt{Comprehensible} and \texttt{Grammar} measure if the summary is fluent and grammatically correct, while \texttt{Repetition} and \texttt{conciseness} measure if there are NO repetitions and the summary is concise. The $\Pr(\text{Yes})$ metric corresponds to the average probability over the dataset that the SEAHORSE model predicts the answer \textit{Yes} to the corresponding question. While we would not expect our MI  metric to correlate with the Grammar or Comprehensible scores, it should strongly correlate with the Main Ideas and Attribution scores as these are proxies to the information contained in the summary.

    \section{Experimental Results}
    \label{sec:experimental_results}

    \begin{table}\centering
\caption{Common quality estimation metrics correlation with the performance on the downstream classification tasks. Where Sent. analysis stands for sentiment analysis, GPT det. for GPT detector, Topic. for topic classification, Policy for policy classification, Emotion for emotion classification and Emb. for paraphrase embedding.}
\label{tab:correlation_table_ct}
\resizebox{0.5\textwidth}{!}{\begin{tabular}{llrrrrr}
\toprule
 &  & Sent. analysis & GPT det. & Policy & Emotion & Emb. \\
  & Metric &  &  &  &  &  \\
\midrule
 & $I(S;T)$ & 0.63 & 0.59 & 0.58 & 0.56 & 0.81 \\
\cline{1-7}
\multirow[c]{9}{*}{Bas.} & BERTScore Precision & 0.53 & 0.68 & 0.47 & 0.46 & 0.70 \\
 & BERTScore Recall & 0.59 & 0.66 & 0.74 & 0.54 & 0.42 \\
 & BLANC & 0.59 & 0.59 & 0.66 & 0.60 & 0.38 \\
 & SMART1 & 0.55 & 0.63 & 0.47 & 0.47 & 0.28 \\
 & SMART2 & 0.55 & 0.63 & 0.47 & 0.47 & 0.28 \\
 & SMARTL & 0.55 & 0.63 & 0.47 & 0.47 & 0.28 \\
 & \texttt{BARTScore} & 0.51 & 0.73 & 0.42 & 0.48 & 0.62 \\
 & \texttt{BERTScore} & 0.64 & 0.75 & 0.64 & 0.58 & 0.61 \\
 & \texttt{ROUGE-L} & 0.56 & 0.55 & 0.54 & 0.47 & 0.29 \\
\cline{1-7}
\multirow[c]{2}{*}{SH.} & Attribution & 0.49 & 0.71 & 0.37 & 0.49 & 0.62 \\
 & Main ideas & 0.33 & 0.37 & 0.60 & 0.38 & 0.47 \\
\cline{1-7}
\bottomrule
\end{tabular}}
\vspace{-0.5cm}
\end{table}

   \noindent\textbf{Correlation with downstream tasks performance.} In \autoref{tab:correlation_table_ct}, we report the correlation between the different metrics and the expected error rate for different classification tasks and with the dot product for the Paraphrase embedding task. The MI  is competitive with both the common quality estimation metrics such as \texttt{BERTScore} and \texttt{BARTScore}. In addition, we report the correlation of the metrics trained on human judgement from the SEAHORSE benchmark to predict whether the main idea of the text is present in the summary and if all elements of the summary are attributable to the source text.\vspace{1mm}

   \begin{figure}
        \centering
        \includegraphics[width=0.49\textwidth]{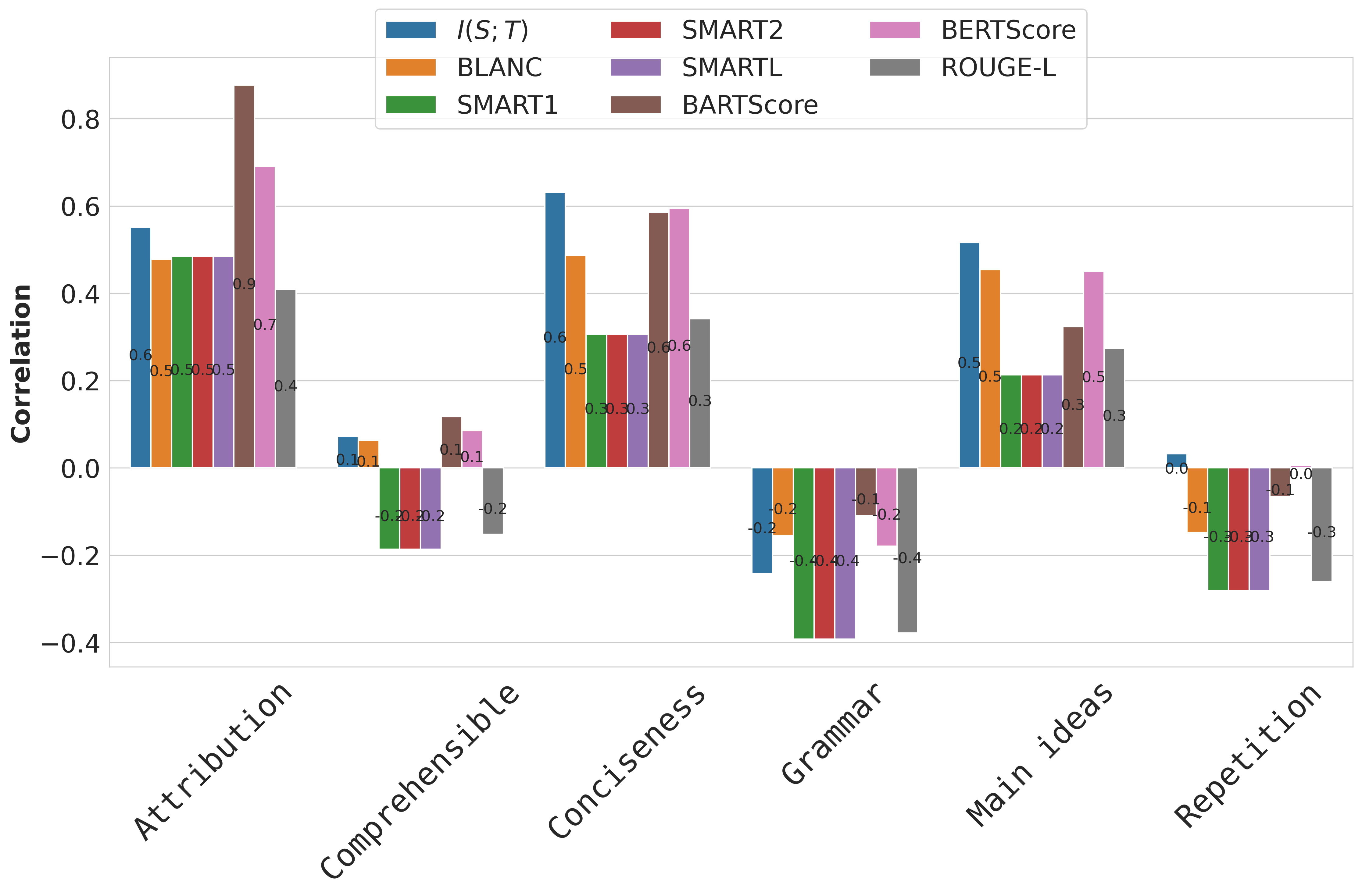}
        \caption{Spearman correlation with human judgment estimated by the SEAHORSE metrics.
        As one would expect the MI does correlate with \texttt{Attribution} and \texttt{Main ideas} but not with \texttt{comprehensible}, \texttt{grammar} or \texttt{repetition}.}
        \label{fig:comprehensive_correlation_seahorse_human_judgement}
        \vspace{-0.5cm}
    \end{figure}

    \noindent\textbf{Correlation with human-judgement-based metrics.} As the SEAHORSE metrics are trained to mimic human judgment, we can use them to assess the behavior of the MI. We found consistent and expected correlations with the different SEAHORSE metrics (\autoref{fig:comprehensive_correlation_seahorse_human_judgement}). \textbf{ The mutual information correlates well with \texttt{Attribution} and \texttt{Main Ideas}  but not with \texttt{Comprehensible}, \texttt{Grammar} and \texttt{Repetition}.} This is not surprising as one would expect the mutual information to capture the amount of information in the summary that is attributable to the source text and not the grammatical correctness or fluency of the summary. The high correlation with \texttt{Conciseness} is a rather surprising result. We believe this is because the conciseness correlates with the strength of the summarizer, which correlates with the MI  between the source texts and the summaries. The stronger the language model, the better the source texts' encoding as a summary. It is also plausible that the learned metrics are flawed in some ways, hindering the results.

    We observe that the MI correlates more consistently with these \texttt{Main idea}, \texttt{Attribution} and \texttt{Conciseness} scores than the common quality assessment metrics, e.g. \texttt{BERTScore} and \texttt{BARTScore} (see \autoref{fig:comprehensive_correlation_seahorse_human_judgement}). It suggests that in addition to being a good predictor of the performance of the downstream tasks, MI is also a better predictor of the human judgment of the quality of the summaries. Notably, the MI, a theoretical quantity derived from Shannon's MI, reproduces human judgment expectations without training on human judgment data.\vspace{1mm}

\noindent\textbf{Comparing summarizers with \texttt{COSMIC}.} In \autoref{fig:model_comparison}, we compare the MI of different models and report their size.  We observe that OOD models --- trained on Arxiv or medical data --- perform very poorly, whereas IN distribution models such as BART display significantly higher MI. Interestingly enough, the size of the model does not seem to be a good predictor of MI. In \autoref{sec:summarizers_hierarchy} we explore further use of the mutual information to compare the informativeness of different summarizers between themselves to construct a hierarchy based on their mutual informativeness.

    \begin{figure}
        \includegraphics[width=0.50\textwidth]{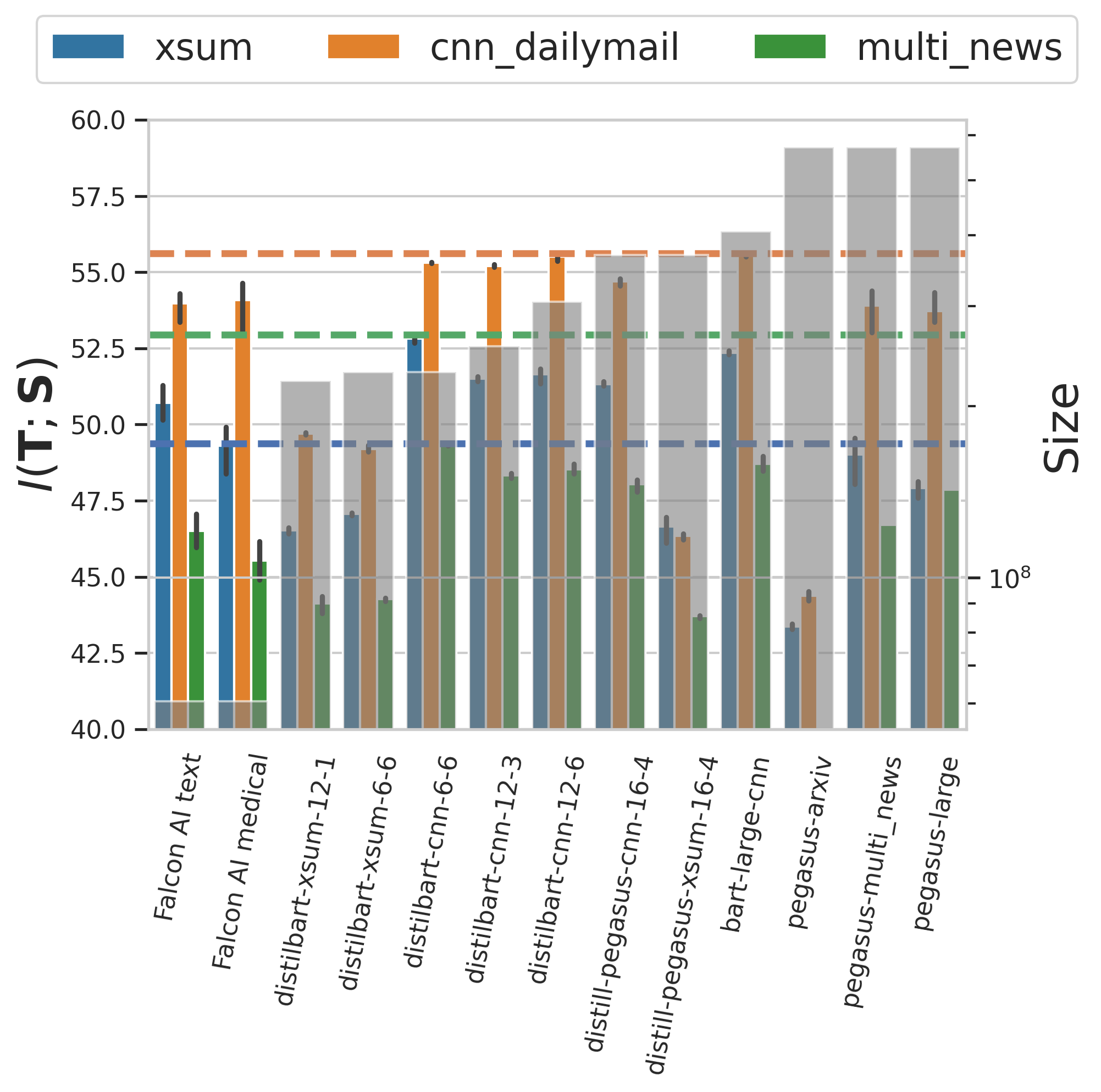}
        \caption{Comparison of models of different sizes and strengths, finetuned on different datasets regarding the MI between their summaries and the source text on different evaluation datasets. The dotted lines highlight the highest MI  for each dataset reached by a model, and the grey area represents the number of parameters of the model}
        \label{fig:model_comparison}
        \vspace{-0.5cm}
    \end{figure}

    \begin{table}
        \centering
        \resizebox{0.35\textwidth}{!}{
        \begin{tabular}{|c|c|c|c|c|}
\hline Metric & Coher. & Cons. & Flu. & Rel. \\
\hline & \multicolumn{4}{|c|}{ Reference Dependent } \\
\hline ROUGE-1 & .35 & .55 & .53 & .58 \\
\hline ROUGE-2 & .23 & .60 & 49 & .43 \\
\hline ROUGE-L & .12 & .12 & .26 & .35 \\
\hline BLEU & .22 & .05 & .33 & .38 \\
\hline CHRF & .35 & .63 & .56 & .55 \\
\hline BERTScore & .33 & -.03 & .14 & .20 \\
\hline MoverScore & .23 & -.05 & .26 & .35 \\
\hline BLEURT & .53 & .20 & .41 & 47 \\
\hline SMS & .27 & 60 & .36 & .40 \\
\hline SMART-1 & .43 & 67 & .64 & .67 \\
\hline SMART-2 & .42 & .75 & .63 & .58 \\
\hline \multirow[t]{2}{*}{ SMART-L } & .57 & .57 & .61 & .73 \\
\hline & \multicolumn{4}{|c|}{ Reference Free } \\
\hline PRISM & .23 & .60 & .36 & .37 \\
\hline T5-ANLI & .25 & .58 & .54 & .52 \\
\hline BARTScore & .35 & .62 & .49 & 45 \\
\hline BARTScore+CNN & .55 & .32 & .59 & .58 \\
\hline$Q^2$ & .25 & .75 & .58 & .45 \\
\hline RISE $_{\text {extMulti-News }}$ & .53 & .73 & .71 & .70 \\
\hline RISE $_{\text {SamSUM }}$ & .53 & .70 & .68 & .70 \\
\hline \multirow[t]{2}{*}{$\operatorname{RISE}_{C N N}$} & .53 & .73 & .75 & .70 \\
\hline & \multicolumn{4}{|c|}{ Ours } \\
\hline $I(\textbf{T}; \textbf{S})$ & .23 & .53 & .47 & .54 \\
\hline

\end{tabular}

        }
        \caption{Comparison of our method against many baselines on the SummEval Human evaluation dataset. We report the system-level Kendall's Tau correlation with human judgments.}
        \label{tab:additional_baselines_human_eval}
        \vspace{-0.05cm}
    \end{table}

    \section{Discussion of Quantitative Results}
    \label{sec:discussion}
    Our findings reveal that various conventional metrics do not consistently align with the effectiveness on downstream tasks. As illustrated in \autoref{fig:radar_chart_correlations_methods}, different metrics exhibit distinct behaviors and correlations with other metrics and the performance of downstream tasks.\vspace{1mm}

\noindent\textbf{SEAHORSE Metrics.} Notably, most SEAHORSE metrics demonstrate limited correlations with the effectiveness of downstream tasks. Unexpectedly, the \texttt{Main idea} metric performs less effectively compared to the \texttt{Attribution} metric.\vspace{1mm}

    \noindent\textbf{\texttt{BERTScore.}} \texttt{BERTScore} displays good correlation with the task-preserving capabilities of the summaries but have poor correlations with human judgements (real or estimated); whereas the mutual information is theoretically grounded and an all-around more consistent metric.

    \begin{figure*}
        \includegraphics[width=\textwidth]{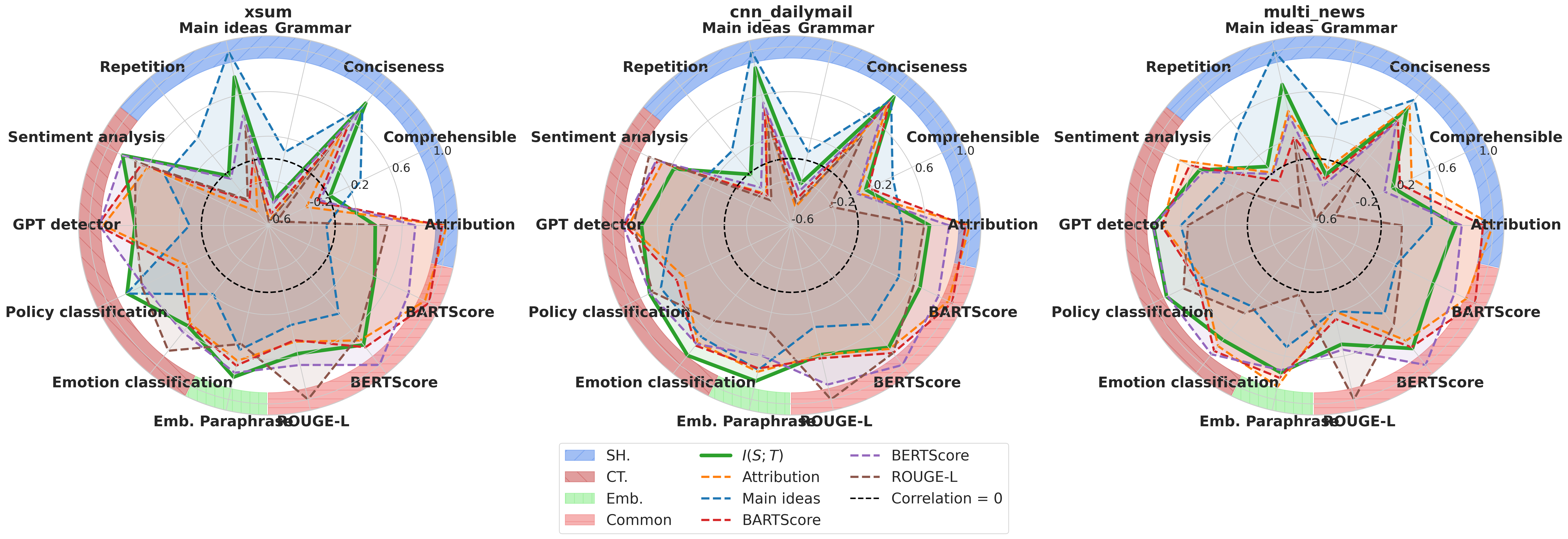}
        \caption{Correlation between the main metrics and different performance metrics for the different datasets. The MI (green) closely follows the behavior of the Attribution metric but is a better predictor of the performance of the downstream tasks (Sentiment analysis, GPT Detector, Topic classification, etc.). By contrast, ROUGE scores do not display consistent correlations. For metrics to be considered effective, they should consistently exhibit positive correlations with each other. In \autoref{sec:table_radar} we report the aggregated numerical results.}
        \label{fig:radar_chart_correlations_methods}
    \end{figure*}

  \begin{figure}[t!]
        \centering
        \includegraphics[width=0.48\textwidth]{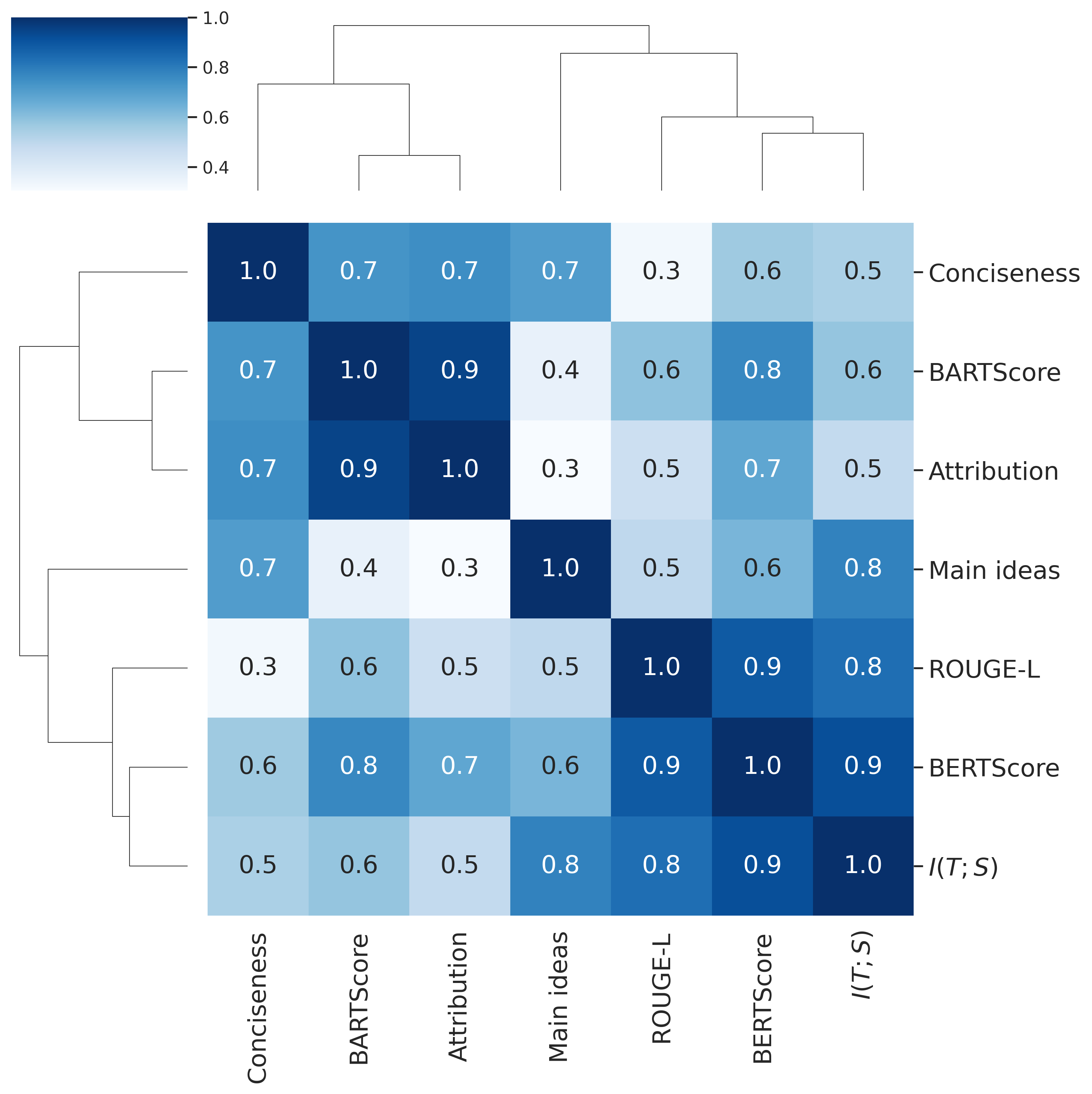}
        \caption{
        The correlation matrix illustrates the relationships among various metrics, with clustering based on their correlation similarity. This clustering indicates the degree of similarity between metrics in terms of their correlation with each other. The goal is to have a diverse set of grounded yet independent metrics that assess different aspects of text summarization quality.
        }
        \label{fig:corr_matrix_}
        \vspace{-0.5cm}
    \end{figure}

    \noindent\textbf{Behavior of MI.} The MI displays very similar behavior to the metric trained to detect whether the ideas presented in the summary come from the source text (Attribution) (\autoref{fig:radar_chart_correlations_methods}), but it surprisingly does not follow the same behavior as the metric trained to detect if the main idea is present (Main idea). Coincidentally, the main idea metric does not correlate with the expected error rate of the classification tasks. This may  be an artifact of the training of SEAHORSE benchmark, a limitation of our metric, or could indicate that answering the question "Is the content of the summary fully attributable to the content of the source text" is more relevant for downstream tasks than "Is the main idea of the source text present in the summary".\vspace{1mm}

    \noindent\textbf{Independence of the metrics.} While it seems that \texttt{Attribution} and the MI follow a similar trajectory in \autoref{fig:radar_chart_correlations_methods}, we found that the MI is not correlated with the \texttt{Attribution} metric (\autoref{fig:corr_matrix_}). This suggests that these two metrics are independent and could be used in conjunction to evaluate summarizers. When it comes to more standard metrics, we find two clusters of metrics that seem to be relatively independent of each other: one represented by \texttt{BERTScore} and comprising the MI and the other represented by \texttt{BARTScore}, which includes \texttt{Attribution}.

    \noindent\textbf{Correlations with Human Judgement.} While the MI correlates less on the SummEval judgment dataset than competitors, such as SMART or RISE, it correlates more with the downstream tasks performance (\autoref{tab:correlation_table_ct}) while not relying on golden summaries for comparison (SMART) or on very large pre-trained models (RISE). As mentioned previously, the comparison with other metrics is not completely fair as the MI is evaluated on unannotated data while the other metrics are evaluated on the (limited) annotated data. Nonetheless, the MI displays on-par performance with the other metrics on the SummEval dataset and is a promising metric for summarizer evaluation overall.

    \section{Summary and Discussion}

    We introduced a task-oriented setting for summary evaluation in which we can derive a principled and clear notion of the quality of a summarizer: the expected risk of performing a task on a summary instead of the source text. We connected this risk theoretically to the mutual information of the source texts and generated summaries and we bounded the risk on both sides using the mutual information. Even if these bounds are task-agnostic and thus potentially loose, we demonstrated experimentally that the mutual information indeed correlated well with the risk. High mutual information indicates that a task performed in the summaries is likely to produce the same outcome as in the source texts. \texttt{COSMIC} is therefore theoretically grounded in a reasonable task-oriented scenario. Our ability to estimate this mutual information practically and its correlation with downstream task performance further underscores its significance.  Our proposed method extends beyond summarization systems and could also contribute to the broader field of multi-modal generation evaluation~\cite{kim2022mutual}.

    \section*{Acknowledgements}
    This work was granted access to the HPC resources of IDRIS under the allocation 2023-AD011013290R2 made by GENCI.

    \section{Limitations \& Ethical considerations}

We have introduced a novel evaluation setting and a theoretically grounded metric for assessing summarizers, yet both have their limitations. Firstly, our setting assumes that the sole objective of a summary is to facilitate downstream task performance, and we define a good summary as one that preserves task outcomes. However, summaries can serve multiple purposes, such as aiding comprehension, acting as educational aids, or promoting the source text, which we do not account for in our approach. %

While mutual information is theoretically grounded, it is not without flaws and fails to capture all nuances of the summarization task. It serves as a tool to evaluate a summarizer's informativeness compared to other metrics lacking theoretical grounding.

It is imperative to use mutual information in conjunction with other metrics to evaluate summaries comprehensively, as it solely addresses the informativeness of summaries about their source text. This metric does not assess grammaticality. Consequently, high mutual information values may arise from imperceptible artefacts that render the summary highly informative about the source text yet unintelligible to human readers.

Moreover, our method indirectly evaluates mutual information in the continuous domain by assessing the mutual information between embeddings generated by a fixed language model. The choice of this model significantly impacts mutual information estimation and the parameters of the estimation tool used.

    \bibliography{custom, maths}

    \appendix

    \onecolumn

    \section{Proof of the Upper and  Lower Bounds}
    \label{sec:full_proof}

\subsection{Proof of the upper bound on the expected error rate}
   We begin by noticing that 
		\begin{align}
1 - \underset{c\in\mathcal{C}}{\sum}p^2(c|\mathbf{s})
& = 1 - \mathbb{E} \left[ p(C|\mathbf{S})  | \mathbf{S} =  \mathbf{s} \right]\nonumber\\
&\geq 1 - \mathbb{E} \left[ \max_{c\in\mathcal{C}} p(c|\mathbf{s}) \big | \mathbf{S} = \mathbf{s} \right]\nonumber\\
& = 1 - \max_{c\in\mathcal{C}} p(c|\mathbf{s}). 
\end{align}
By taking the expectation over $\mathbf{S}$ at both sides and using the well-known relationship with the Bayes error,  we obtain the following  inequality: 
\begin{align}
 P_e (c,\widehat{c},\theta) & =  1 - \mathbb{E}  \left[\max_{c\in\mathcal{C}} p(c|\mathbf{S})  \right]\nonumber\\ 
 &\leq 1-  \mathbb{E}  \left[  \underset{c\in\mathcal{C}}{\sum}p^2(c|\mathbf{S}) \right]. \label{eq-mis1}
\end{align}
Similarly, it is possible to derive a lower bound: 
	\begin{align}
\underset{c\in\mathcal{C}}{\sum}p^2(c|\mathbf{s}) 
& =  p^2(c^\star|\mathbf{s})  + \underset{c\neq c^\star }{\sum} p^2(c|\mathbf{s}) \geq  \left( \max_{c\in\mathcal{C}} p(c|\mathbf{s})\right)^2,  \label{eq-mis2}
\end{align}
where $c^\star(\mathbf{s}) =\arg\max_{c\in\mathcal{C}} p(c|\mathbf{s}) $. By taking the expectation over $\mathbf{S}$ at both sides, we obtain: 
	\begin{align}
\sqrt{ \mathbb{E}\left[ \underset{c\in\mathcal{C}}{\sum}p^2(c|\mathbf{S}) \right]  } 
& \geq  
\mathbb{E}\left[\sqrt{ \underset{c\in\mathcal{C}}{\sum}p^2(c|\mathbf{S}) } \right] \nonumber\\ 
& \geq \mathbb{E}\left[  \max_{c\in\mathcal{C}} p(c|\mathbf{S})   \right]  =  1-   P_e (c,\widehat{c},\theta).  \label{eq-mis3}
 \end{align}
Let us denote the second order R\'enyi's entropy \cite{Rnyi1961OnMO} conditioned on $\mathbf{s}$ as follow: 
\begin{equation}
H_2(C|\mathbf{s}) \triangleq -\frac{1}{2} \log \left( \sum\limits_{c\in\mathcal{C}} p^2 (c|\mathbf{s}) \right),  
\end{equation}
and thus, 
\begin{equation}
\sum\limits_{c\in\mathcal{C}} p^2 (c|\mathbf{s}) = \exp\left(  -2 H_2(C|\mathbf{s}) \right). \label{eq-renyi}  
\end{equation}
By replacing \eqref{eq-renyi} in \eqref{eq-mis1} and in \eqref{eq-mis2}, we obtain 
\begin{align}
1 - \sqrt{ \displaystyle \mathbb{E}_{\mathbf{s} \sim  p_S } \left[ \exp\left(  -2 H_2(C|\mathbf{s}) \right)  \right] }  \leq  P_e (c,\widehat{c},\theta)  \leq 1 - \mathbb{E}_{\mathbf{s} \sim  p_S } \left[ \exp\left(  -2 H_2(C|\mathbf{s}) \right)  \right]. \label{eq-upper-bound}
\end{align}
Since $\log$ is a concave function:  
\begin{equation}
\log \left( \sum\limits_{c\in\mathcal{C}} p^2 (c|\mathbf{s})\right)  \geq   \sum\limits_{c\in\mathcal{C}} p(c|\mathbf{s})  \log p(c|\mathbf{s}),  
\end{equation}
we have that 
\begin{equation}
H_2(C|\mathbf{s}) \leq -\sum\limits_{c\in\mathcal{C}} p(c|\mathbf{s})  \log p(c|\mathbf{s}) \triangleq  H(C|\mathbf{s}), \label{eq-shannon-inequality}
\end{equation}
where $H(C|\mathbf{s})$ indicates the Shannon entropy conditioned to the given observation $\mathbf{s}$. Replacing \eqref{eq-shannon-inequality} in the upper bound of \eqref{eq-upper-bound} yields 
\begin{align}
 P_e (c,\widehat{c},\theta)  & \leq 1 - \mathbb{E}_{\mathbf{s} \sim \mathbf{S}} \left[ \exp\left(  - H(C|\mathbf{s})  \right)  \right] \nonumber\\
 & \leq 1 - \exp\left(  - H(C|\mathbf{S})  \right) ,  \label{eq-markov1}\\
  & \leq  1 - \exp\left(  - H(\mathbf{T} | \mathbf{S})  \right), \label{eq-markov2} \\
    & \equiv  1 -  \kappa  \exp\big(  I(\mathbf{T} ; \mathbf{S}) \big), \label{eq-markov3} 
\end{align}
where \eqref{eq-markov1} follows from the fact that the negative exponential function is convex ;   \eqref{eq-markov2} follows by Data-Processing Inequality  \cite{cover91} since $ C \triangleq  c(\mathbf{T})$ and  thus $ C   \leftrightarrow \mathbf{T} \leftrightarrow \mathbf{S}$ form a Markov Chain and $H(\mathbf{T} | \mathbf{S}) $ denotes the differential entropy of the text $\mathbf{T}$  given the summary $\mathbf{S}$ ; and \eqref{eq-markov3} follows by an appropriate definition of the constant $0< \kappa <1$ which does not depend on the summary random variable $\mathbf{S}$. This concludes the proof of the desired upper bound.

    \subsection{Review of the Distortion-Rate Function}
\label{sec:RD-function}

    The rate-distortion (RD) function of a random variable $C$ for a given distortion function $\ell(\cdot, \cdot)$ is defined as \cite[eq.~(1.4)]{csiszar74}
    \begin{align}
        R_{C,\ell}(D) \,\triangleq \inf_{\rule{0mm}{4.3mm}\substack{p(\widehat{c}|c):\\ \Esp[\ell(C,\widehat{C})] \leq D}} \!\! I(C;\widehat{C}).\label{eq-RD-def}
    \end{align}
    For convenience, we assume that
    \[ \inf_{\widehat{c}}\ell(c,\widehat{c})=0, \ \ \forall c. \]
    Furthermore, we suppose that there exists $D>0$ such that $R_{C,\ell}(D)$ is finite. We denote the infimum of those $D$ by $D_{\min}$ and $R_{\max}\triangleq R_{C,\ell}(D_{\min})$ (or, more precisely, $R_{\max}\triangleq \lim_{D\to D_{\min}+}R(D)$).

    The following properties (see \cite[Lem.~1.1]{csiszar74}) of the RD function will be used.
    \begin{thm}
        \label{thm:rdprop}
        The RD function $R_{C,\ell}(D)$ is a non-increasing convex function of $D$ on the interval $(D_{\min}, \infty)$.
        It is monotonically decreasing on the interval $(D_{\min},D_{\max})$ and constant with $R_{C,\ell}(D)=R_{\min}$ on $[D_{\max},\infty)$ (here $D_{\max}=\infty$ and $D_{\min}=0$ are possible). \textbf{The inverse function $\RD^{-1}(r)$ is well defined on $(R_{\min}, R_{\max})$ and is monotonically decreasing. It is known as the distortion rate (DR)
            function of the random variable $C$ for the given distortion function $\ell(\cdot, \cdot)$.}
    \end{thm}

    \subsection{Proof of the lower bound on the average  of a general  loss}
    \label{sec:bound_proof}
   For any suitably  loss or evaluation metric denoted by $\ell(c, \widehat{c})$, the quality of the predicted concept $\widehat{c}(\mathbf{S}) $, which is  based on the random  summary $\mathbf{S}$, compared to a desired target concept ${C}\triangleq {c}(\mathbf{T})$ from the original text $\mathbf{T}$,   can be expressed by the average loss $\Esp[\ell(c(\mathbf{T}),\widehat{c}(\mathbf{S}))]$ with respect to the joint distribution of the source text and its summary $(\mathbf{T},\mathbf{S})$. 
   
From Data-Processing Inequality and the definition of the RD function \eqref{eq-RD-def}, the following proposition provides a lower bound on the performance of any arbitrary  predictor $\widehat{c}(\mathbf{S})$ of the target concept $C$:

        \begin{align}
                     \label{eq:boundmi}
                I(\mathbf{T};\mathbf{S}) & \geq  I(c(\mathbf{T});\widehat{c}(\mathbf{S})) \\
                & \geq \inf_{\substack{p(\widehat{c}|c)\,:\, \label{eq:boundmi2}\\ 
                \Esp[\ell(C,\widehat{C})] \leq  \Esp[\ell(c(\mathbf{T}), \widehat{c}(\mathbf{S}))] }}  I(C;\widehat{C}) \\
                & = \RD\left( \Esp[\ell(c(\mathbf{T}), \widehat{c}(\mathbf{S}))] \right), \label{eq:boundmi3}
         \end{align}
 where the inequality in \eqref{eq:boundmi} follow from Data-Processing since $ c(\mathbf{X})   \leftrightarrow \mathbf{X} \leftrightarrow \mathbf{S} \leftrightarrow \widehat{c}(\mathbf{S})$ form a Markov Chain ; and  \eqref{eq:boundmi2} follows by the definition of the RD function \eqref{eq-RD-def}.   

\begin{itemize}
    \item         For $\Esp[\ell(c(\mathbf{T}), \widehat{c}(\mathbf{S}))] \in (D_{\min},D_{\max})$, we can invert the RD function \eqref{eq-RD-def}, and thus we obtain from \eqref{eq:boundmi} the fundamental bound $\RD^{-1}\left( I(\mathbf{T}; \mathbf{S})  \right)\leq \Esp[\ell(c(\mathbf{T}), \widehat{c}(\mathbf{S}))]$
        or, equivalently,
        \begin{align}
            \label{eq:bounded}
           \Esp[\ell(c(\mathbf{T}), \widehat{c}(\mathbf{S}))]  \,\ge \RD^{-1}\big(I(\mathbf{T};\mathbf{S}) \big),  
        \end{align}
which holds for any predictor $ \widehat{c}(\mathbf{S})$ and thus, for the one minimizing the left-hand size of \eqref{eq:bounded}. 

\item     For $\Esp[\ell(c(\mathbf{T}), \widehat{c}(\mathbf{S}))]  < D_{\min}$ equation~\eqref{eq:boundmi} reduces to $I(\textbf{T};\textbf{S}) \geq +\infty$ which shows that to achieve an expected distortion below $D_{\min}$ the random variables $(\mathbf{T},\mathbf{S})$ must have a joint distribution that is not absolutely continuous with respect to the product of their marginal distributions.

\item    For $\Esp[\ell(c(\mathbf{T}), \widehat{c}(\mathbf{S}))]  \geq D_{\max}$ we obtain the trivial bound $ I(\mathbf{T};\mathbf{S}) \geq 0$.
   \end{itemize}
    
    \begin{rem}
     Inequality~\eqref{eq:bounded} shows that for arbitrary random concept $c(\mathbf{T})$ about the text to be inferred with $ \widehat{c}(\mathbf{S})$ using the random summary $\mathbf{S}$ generated from $\mathbf{T}$, the expected loss  of any predictor  $\widehat{c}(\cdot)$ is lower bounded by a monotonically decreasing function of the mutual information between $\mathbf{T}$ and $\mathbf{S}$. \textbf{Thus, irrespective of the precise formulation of the loss function or the task defined by $c(\cdot)$ for execution on the summary, maximizing the mutual information $I(\mathbf{T};\mathbf{S})$ stands as a requisite condition for achieving commendable inference performance.
} Our result suggests that estimating mutual information of a given summarizer can be a good proxy to assess its quality in the sense of preserving relevant information about concepts.      
    \end{rem}
    
      \section{Rate-Distortion Bound for Classification Tasks}
    \label{sec:ap_rate_distortion}

    We assume that $C$ is uniformly distributed on the finite set $\mathcal{C}\triangleq \{1,2,\dots, m\}$, and we use the Hamming distortion function defined by
    \[
        \ell(c,\widehat{c})\triangleq \begin{cases}
                                                0, & \text{if }\, c=\widehat{c} \\
                                                1, & \text{else}\,.
        \end{cases}
    \]
    Note that the expected distortion equals the expected error rate, i.e.,
    \begin{equation}
        \label{eq:exp_HD}
      P_e (c,\widehat{c},\theta) \,\triangleq\, \inf_{\widehat{c}: \,\Omega^* \rightarrow  \, \mathcal{C}} \mathrm{Pr} \big(C \neq \widehat{c}(\textbf{S}) \big)   \,=\,   \inf_{\widehat{c}: \,\Omega^* \rightarrow  \, \mathcal{C}}\Esp[\ell\big(C,\widehat{c}(\textbf{S})\big)]  \,.
    \end{equation}
    The RD function is given by \cite[Problem~10.5]{cover91} (not solved there)
    \begin{equation}
        \label{eq:RD_fano}
        R(D) \,= \begin{cases}
                     \log m - H_b(D) - D\log(m-1), & \text{if }\, 0\leq D \leq 1-\frac{1}{m} \\[1mm]
                     0, & \text{if }\, 1-\frac{1}{m}< D\,.
        \end{cases}
    \end{equation}
    Here, $H_b(D)$ is the binary entropy function of $D$, i.e., $H_b(D) \triangleq  - D \log (D) - (1-D) \log (1-D)$.
    Inserting \eqref{eq:exp_HD} and \eqref{eq:RD_fano} into \eqref{eq:boundmi}, we obtain
    \[
        I(\mathbf{T};\mathbf{S}) \,\geq\, \log m - H_b\big( P_e (c,\widehat{c},\theta) \big) - P_e (c,\widehat{c},\theta) \log(m-1) \,,
    \]
    which is the well known Fano's inequality \cite[Th.~2.10.1]{cover91}. This can be put into the form of the general lower bound \eqref{eq:bounded}:
    \[
       P_e (c,\widehat{c},\theta) \,\ge\, \RDH^{-1}\big( I(\mathbf{T};\mathbf{S})  \big) \,.
    \]
    However, a closed-form expression of $\RDH^{-1}(I)$ is not available.
    The function is plotted for different $m$ in Figure~\ref{fig:fano}.
    \begin{figure}[thb]
        \centering
        \begin{tikzpicture}
            \begin{axis}
                [
                xmin=0,
                ymin=0,
                xlabel={$I$},
                ylabel={$\RDH^{-1}(I)$}
                ]
                \addplot[
                    red,
                    domain=0:0.5,
                    samples=50
                ]
                ({ln(2) + x*ln(x)+(1-x)*ln(1-x)-x*ln(2-1)},x);
                \addplot[
                    green,
                    samples=50,
                    domain=0:0.75
                ]
                ({ln(4) + x*ln(x)+(1-x)*ln(1-x)-x*ln(4-1)},x);
                \addplot[
                    blue,
                    samples=50,
                    domain=0:0.9
                ]
                ({ln(10) + x*ln(x)+(1-x)*ln(1-x)-x*ln(10-1)},x);
                \legend{$m=2$,$m=4$ , $m=10$}
            \end{axis}
        \end{tikzpicture}
        \caption{$R(I)$ for different values of $m$.}
        \label{fig:fano}
    \end{figure}
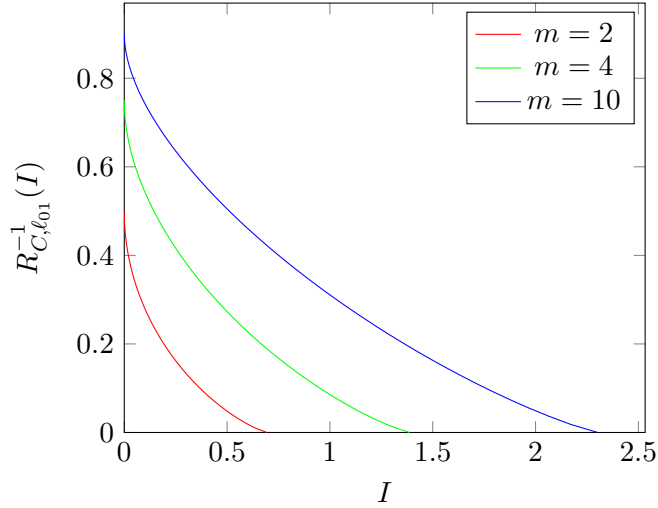

    \section{Model Specifications}
    \label{sec:models_app}

    All the evaluated models are listed with their characteristics in \autoref{tab:model_summaries}.
    
    \begin{table*}[h!]\centering
\caption{Summary of the models we benchmarked with their name on the Huggingface hub, size and performance metrics.}
\label{tab:model_summaries}
\resizebox{\textwidth}{!}{\begin{tabular}{lllrrrrrrrrrrrr}
\toprule
 &  & Size & \texttt{ROUGE-L} & \texttt{BERTScore} & \texttt{BARTScore} & M. I. & Attr. & Rep. & Compr. & Conc. & Gram. & $I(T,S)$ & H(T|S) & H(S|T) \\
Model & Dataset &  &  &  &  &  &  &  &  &  &  &  &  &  \\
\midrule
\multirow[c]{3}{*}{Falconsai/medical\_summarization} & cnn\_dailymail & 60 M & 0.17 & 0.16 & -1.70 & 0.33 & 0.79 & 0.46 & 0.76 & 0.42 & 0.30 & 54.09 & -14.60 & -7.24 \\
 & multi\_news & 60 M & 0.09 & 0.02 & -2.06 & 0.22 & 0.74 & 0.55 & 0.70 & 0.36 & 0.34 & 45.53 & -11.49 & -2.29 \\
 & xsum & 60 M & 0.31 & 0.24 & -1.69 & 0.33 & 0.77 & 0.36 & 0.74 & 0.33 & 0.34 & 49.30 & -15.25 & -6.26 \\
\cline{1-15}
\multirow[c]{3}{*}{Falconsai/text\_summarization} & cnn\_dailymail & 60 M & 0.15 & 0.17 & -1.48 & 0.36 & 0.82 & 0.64 & 0.81 & 0.44 & 0.38 & 53.97 & -14.49 & -8.22 \\
 & multi\_news & 60 M & 0.10 & 0.06 & -1.67 & 0.23 & 0.80 & 0.60 & 0.75 & 0.38 & 0.38 & 46.50 & -12.46 & -4.00 \\
 & xsum & 60 M & 0.31 & 0.29 & -1.53 & 0.33 & 0.81 & 0.58 & 0.77 & 0.35 & 0.37 & 50.71 & -16.67 & -8.84 \\
\cline{1-15}
\multirow[c]{3}{*}{sshleifer/distilbart-xsum-12-1} & xsum & 221 M & 0.08 & 0.03 & -3.27 & 0.29 & 0.35 & 0.86 & 0.90 & 0.22 & 0.80 & 46.52 & -12.49 & 17.01 \\
 & multi\_news & 221 M & 0.03 & -0.09 & -3.08 & 0.17 & 0.42 & 0.82 & 0.77 & 0.23 & 0.68 & 44.13 & -10.10 & 16.22 \\
 & cnn\_dailymail & 221 M & 0.04 & -0.04 & -3.16 & 0.21 & 0.41 & 0.87 & 0.83 & 0.22 & 0.67 & 49.70 & -10.20 & 15.21 \\
\cline{1-15}
\multirow[c]{3}{*}{sshleifer/distilbart-cnn-6-6} & cnn\_dailymail & 229 M & 0.16 & 0.22 & -1.35 & 0.47 & 0.85 & 0.92 & 0.85 & 0.55 & 0.50 & 55.31 & -15.80 & -11.84 \\
 & multi\_news & 229 M & 0.10 & 0.15 & -1.37 & 0.34 & 0.86 & 0.90 & 0.84 & 0.52 & 0.57 & 49.34 & -15.31 & -8.99 \\
 & xsum & 229 M & 0.32 & 0.35 & -1.30 & 0.47 & 0.84 & 0.91 & 0.90 & 0.49 & 0.60 & 52.81 & -18.77 & -13.27 \\
\cline{1-15}
\multirow[c]{3}{*}{sshleifer/distilbart-xsum-6-6} & cnn\_dailymail & 229 M & 0.04 & -0.02 & -2.88 & 0.38 & 0.55 & 0.96 & 0.96 & 0.41 & 0.87 & 49.19 & -9.68 & 12.37 \\
 & multi\_news & 229 M & 0.03 & -0.05 & -2.75 & 0.33 & 0.57 & 0.95 & 0.94 & 0.44 & 0.88 & 44.26 & -10.23 & 10.02 \\
 & xsum & 229 M & 0.09 & 0.06 & -2.95 & 0.46 & 0.46 & 0.97 & 0.97 & 0.38 & 0.91 & 47.07 & -13.02 & 13.67 \\
\cline{1-15}
\multirow[c]{3}{*}{sshleifer/distilbart-cnn-12-3} & xsum & 255 M & 0.31 & 0.32 & -1.58 & 0.49 & 0.74 & 0.82 & 0.94 & 0.47 & 0.64 & 51.50 & -17.45 & -11.65 \\
 & cnn\_dailymail & 255 M & 0.18 & 0.24 & -1.36 & 0.48 & 0.82 & 0.87 & 0.93 & 0.54 & 0.60 & 55.21 & -15.71 & -13.15 \\
 & multi\_news & 255 M & 0.10 & 0.13 & -1.49 & 0.36 & 0.82 & 0.86 & 0.92 & 0.53 & 0.66 & 48.32 & -14.28 & -9.03 \\
\cline{1-15}
\multirow[c]{3}{*}{sshleifer/distilbart-cnn-12-6} & cnn\_dailymail & 305 M & 0.18 & 0.25 & -1.31 & 0.52 & 0.84 & 0.91 & 0.93 & 0.58 & 0.62 & 55.51 & -16.00 & -13.79 \\
 & multi\_news & 305 M & 0.11 & 0.15 & -1.40 & 0.39 & 0.84 & 0.89 & 0.93 & 0.57 & 0.69 & 48.54 & -14.51 & -9.60 \\
 & xsum & 305 M & 0.31 & 0.33 & -1.50 & 0.54 & 0.76 & 0.91 & 0.96 & 0.51 & 0.71 & 51.65 & -17.58 & -12.11 \\
\cline{1-15}
\multirow[c]{3}{*}{sshleifer/distill-pegasus-xsum-16-4} & xsum & 369 M & 0.08 & 0.05 & -2.88 & 0.44 & 0.45 & 0.97 & 0.97 & 0.36 & 0.91 & 46.65 & -12.57 & 16.31 \\
 & multi\_news & 369 M & 0.03 & -0.06 & -2.60 & 0.28 & 0.56 & 0.95 & 0.93 & 0.39 & 0.86 & 43.70 & -9.68 & 16.49 \\
 & cnn\_dailymail & 369 M & 0.04 & -0.04 & -2.92 & 0.27 & 0.50 & 0.96 & 0.95 & 0.32 & 0.85 & 46.35 & -6.84 & 13.27 \\
\cline{1-15}
\multirow[c]{3}{*}{sshleifer/distill-pegasus-cnn-16-4} & xsum & 369 M & 0.23 & 0.28 & -1.49 & 0.47 & 0.81 & 0.84 & 0.94 & 0.49 & 0.70 & 51.32 & -17.29 & -6.70 \\
 & multi\_news & 369 M & 0.08 & 0.11 & -1.44 & 0.33 & 0.84 & 0.82 & 0.90 & 0.51 & 0.69 & 48.03 & -13.99 & -4.63 \\
 & cnn\_dailymail & 369 M & 0.13 & 0.19 & -1.40 & 0.48 & 0.83 & 0.81 & 0.92 & 0.54 & 0.64 & 54.69 & -15.19 & -8.70 \\
\cline{1-15}
\multirow[c]{3}{*}{facebook/bart-large-cnn} & cnn\_dailymail & 406 M & 0.18 & 0.25 & -1.19 & 0.49 & 0.86 & 0.95 & 0.94 & 0.59 & 0.67 & 55.54 & -16.05 & -13.55 \\
 & multi\_news & 406 M & 0.10 & 0.15 & -1.30 & 0.37 & 0.86 & 0.95 & 0.94 & 0.58 & 0.73 & 48.71 & -14.69 & -9.45 \\
 & xsum & 406 M & 0.32 & 0.34 & -1.31 & 0.52 & 0.81 & 0.96 & 0.97 & 0.53 & 0.74 & 52.36 & -18.33 & -13.06 \\
\cline{1-15}
\multirow[c]{2}{*}{google/pegasus-multi\_news} & multi\_news & 570 M & 0.06 & 0.02 & -2.51 & 0.48 & 0.69 & 0.93 & 0.79 & 0.52 & 0.58 & 46.71 & -12.68 & -4.24 \\
 & xsum & 570 M & 0.27 & 0.20 & -2.58 & 0.53 & 0.49 & 0.94 & 0.85 & 0.39 & 0.72 & 49.01 & -15.00 & -10.98 \\
\cline{1-15}
google/pegasus-arxiv & xsum & 570 M & 0.14 & -0.22 & -3.52 & 0.11 & 0.32 & 0.25 & 0.43 & 0.13 & 0.28 & 43.38 & -9.33 & 0.79 \\
\cline{1-15}
\multirow[c]{3}{*}{google/pegasus-large} & cnn\_dailymail & 570 M & 0.20 & 0.25 & -1.21 & 0.26 & 0.87 & 0.66 & 0.84 & 0.43 & 0.51 & 53.73 & -14.21 & -9.28 \\
 & multi\_news & 570 M & 0.07 & 0.12 & -1.52 & 0.18 & 0.82 & 0.82 & 0.72 & 0.37 & 0.42 & 47.86 & -13.81 & -5.25 \\
 & xsum & 570 M & 0.27 & 0.31 & -1.19 & 0.24 & 0.88 & 0.86 & 0.89 & 0.37 & 0.70 & 47.92 & -13.89 & -4.18 \\
\cline{1-15}
google/pegasus-multi\_news & cnn\_dailymail & 570 M & 0.16 & 0.16 & -2.26 & 0.47 & 0.63 & 0.94 & 0.83 & 0.43 & 0.65 & 53.89 & -14.39 & -13.12 \\
\cline{1-15}
\multirow[c]{2}{*}{google/pegasus-arxiv} & cnn\_dailymail & 570 M & 0.10 & -0.26 & -3.30 & 0.08 & 0.35 & 0.26 & 0.45 & 0.12 & 0.30 & 44.37 & -4.86 & -0.28 \\
 & multi\_news & 570 M & 0.05 & -0.27 & -3.59 & 0.06 & 0.32 & 0.36 & 0.46 & 0.12 & 0.39 & 39.87 & -5.82 & 1.21 \\
\cline{1-15}
\bottomrule
\end{tabular}}
\end{table*}

    \section{Mutual Information Estimation with KNIFE}

    \subsection{Predictive mutual information}

    The estimation of mutual information is widely acknowledged to be challenging, and in practical scenarios, we often resort to approximating it with a proxy measure known as Arimoto information~\cite{arimoto1971information} or recently rediscovered as predictive mutual information~\cite{xu2020theory}.
    Instead of computing the mutual information in the general case, it is computed under computational constraints enforced by a class of predictive functions.

    \begin{defn}[Predictive conditional entropies]
    \label{def:predictive_conditional_entropies}
    Let $\mathbf{T}$ and $\mathbf{S}$ be two random variables respectively over $\Omega^*$ and a class of functions $\mathcal{F}$:
        \[
            \begin{aligned}
                & h_{\mathcal{F}}(\mathbf{T} \mid \mathbf{S})=\inf _{f \in \mathcal{F}} \mathbb{E}_{\mathbf{T}\mathbf{S}}[-\log f_{[\mathbf{S}]}(\mathbf{T})], \\
                & h_{\mathcal{F}}(\mathbf{T} \mid \varnothing)=\inf _{f \in \mathcal{F}} \mathbb{E}_{\mathbf{T}}[-\log f_{[\varnothing]}(\mathbf{T})].
            \end{aligned}
        \]
    \end{defn}
    For the sake of brevity, we denote $h_{\mathcal{F}}(\mathbf{T} \mid \varnothing)$ by $h_{\mathcal{F}}(\mathbf{T})$.

    \begin{defn}[Predictive $\mathcal{F}$-information]
    \label{def:predictive_mutual_information}
            \begin{align}
                 I_{\mathcal{F}}(\mathbf{S} \rightarrow \mathbf{T}) & \triangleq  h_{\mathcal{F}}(\mathbf{T})-h_{\mathcal{F}}(\mathbf{T} \mid \mathbf{S}), \\
                I_{\mathcal{F}}(\mathbf{T} \rightarrow \mathbf{S}) & \triangleq h_{\mathcal{F}}(\mathbf{S})-h_{\mathcal{F}}(\mathbf{S} \mid \mathbf{T}). 
            \end{align}
       If $\mathcal{F}$ represents the set of all possible functions, then we expect $I_{\mathcal{F}}(\mathbf{T} \rightarrow \mathbf{S}) = I_{\mathcal{F}}(\mathbf{S} \rightarrow \mathbf{T}) = I_{\mathcal{F}}(\mathbf{S}; \mathbf{T})$. However, due to computational limitations imposed by $\mathcal{F}$, these estimators are not symmetrical. Therefore, we have two options for estimating the mutual information.
    \end{defn}

    \begin{rem}
       The predictive mutual information is asymmetric with respect to $\mathbf{S}$ and $\mathbf{T}$. In our context, we opt for using the predictive mutual information $I_{\mathcal{F}}(\mathbf{S} \rightarrow \mathbf{T})$ to gauge the degree of information preservation about the source texts through the summarization process. Thus, we define $\widehat{I}(\mathbf{T}; \mathbf{S}) \equiv I_{\mathcal{F}}(\mathbf{S} \rightarrow \mathbf{T})$. Experimentally, we observed that the predictive mutual information $I_{\mathcal{F}}(\mathbf{T} \rightarrow \mathbf{S})$ did not yield consistent outcomes. Further details are provided in \autoref{sec:full_proof}. We leverage this asymmetry in \autoref{sec:summarizers_hierarchy}.
    \end{rem}
\vspace{1mm}

    \noindent\textbf{Mutual information estimator.} We utilize the KNIFE estimator~\cite{pichler2022differential} to estimate the predictive mutual information between continuous random variables. This estimator defines $\mathcal{F}$ as the class of Gaussian Mixtures with $K$ modes, introducing a soft-clustering approach for text generation evaluation~\cite{pillutla2021mauve, pimentel2023usefulness}. Through experimentation, we observed that varying the number of modes $K$ did not notably affect the results. Therefore, we present our findings with $K=4$.

    \label{sec:gaussian_vs_knife}

    \label{sec:ap_mi_estimation}
    Initially, we examine the applicability of the basic Gaussian estimator of mutual information proposed in~\cite{kim2022mutual}. However, we find it unsuitable for our scenario since the embeddings of both the source texts and summaries exhibit multimodal distributions, as illustrated in \autoref{fig:non_gaussian_datasets}. Instead, we find the KNIFE estimator~\cite{pichler2022differential} to be better suited for our context as it is designed to estimate mixtures of Gaussians.

    \begin{figure}
        \centering
        \includegraphics[width=1\textwidth]{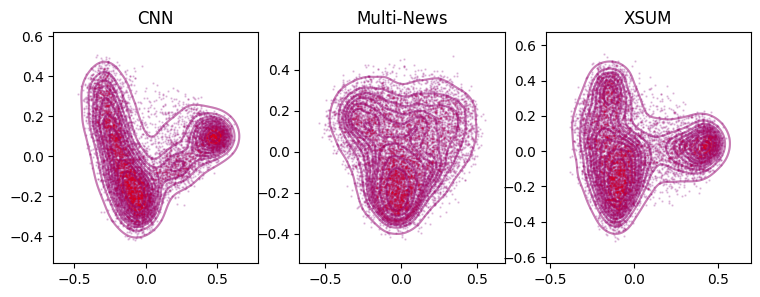}
        \caption{PCA was performed on the embeddings of the source texts and summaries for three datasets under consideration. It is evident from the plots that the embeddings do not exhibit a Gaussian distribution but rather resemble a mixture of Gaussians. This characteristic makes the Gaussian estimator of MI unsuitable for our purposes.}
        \label{fig:non_gaussian_datasets}
    \end{figure}

    \section{Comparison Across Datasets}
    \label{sec:seahorse_across}
    \label{sec:seahorse_app}

   One might seek to compare models evaluated on different datasets. However, the mutual information, in its current form, is not suitable for such comparisons as it relies on the entropy of the dataset. In \autoref{fig:correlation_with_human_judgement_htext}, we demonstrate that for the SEAHORSE benchmark, there are significant variations in the entropies of the source texts, introducing bias into the comparison of mutual informations estimated as $I(\mathbf{T} ; \mathbf{S}) = H(\mathbf{T}) - H(\mathbf{T}| \mathbf{S})$. To address this issue, we propose normalizing the mutual information by the entropy of the dataset. We show that the normalized version of the mutual information correlates with the responses to the questions of the SEAHORSE benchmark (cf. \autoref{fig:correlation_with_human_judgment}).\vspace{1mm}

    \begin{rem}
        While the SEAHORSE benchmarks contains similar texts for all their models, the models have been each evaluated on different samples. For instant, in the english subset, only $91$ samples out of the $10 000$ are common to all models. This leads to biaises in the evaluation of the mutual information.
    \end{rem}

   We observed that variations in the source text datasets significantly affect the estimation results of mutual information. There's a tendency for smaller models to exhibit higher source text entropies, leading to misleading comparisons. To address this, we suggest computing the normalized MI between the source texts and summaries. This normalized MI is defined as follows:
    \begin{align}
        \label{eq:norm_mi}
       \textrm{Normalized MI} \triangleq  \frac{ I(\mathbf{T};\mathbf{S}) }{H(\mathbf{T})} = 1 - \frac{ H(\mathbf{T}|\mathbf{S}) }{H(\mathbf{T})},
    \end{align}
    where $H(\mathbf{T}|\mathbf{S}) \leq  H(\mathbf{T}) $.

    In \autoref{fig:correlation_with_human_judgment}, we observed weak correlations between this normalized MI and the human judgments reported in the SEAHORSE benchmark. This discrepancy might arise from the evaluation of different datasets for each model, suggesting that the normalized MI might not be the most suitable normalization method. Hence, comparing models evaluated on different datasets should be avoided for now. However, when evaluated on the same datasets, MI correlates well with the metrics trained on the SEAHORSE benchmark. This indicates that MI is an promising tool for evaluating summarizers.

    \begin{figure}
        \centering
        \includegraphics[width=1\textwidth]{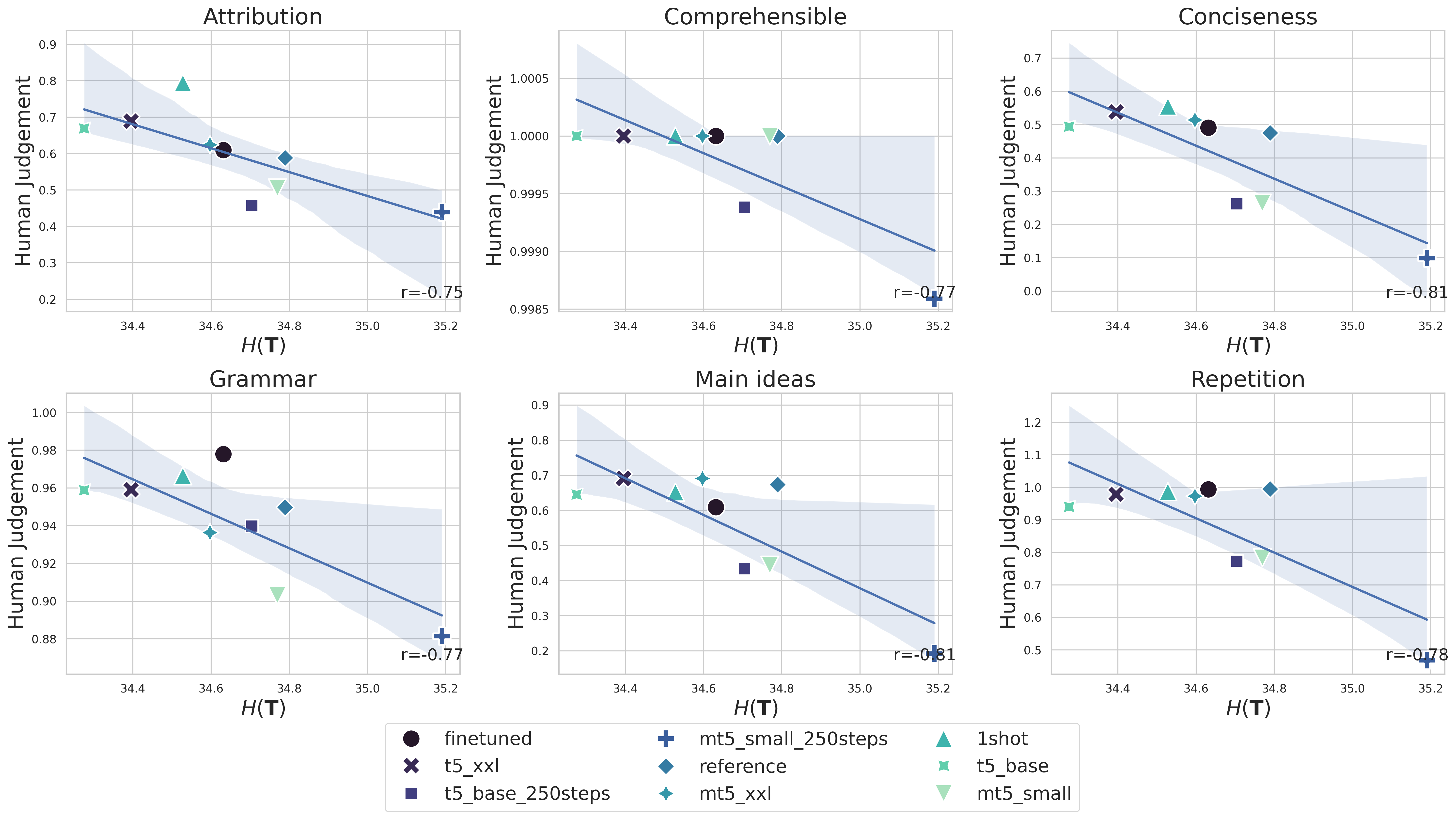}
        \caption{Correlation between the source texts entropies as estimated per KNIFE and the answers to the SEAHORSE benchmark. We observe that the entropy of the source texts correlates negatively with the answers.}
        \label{fig:correlation_with_human_judgement_htext}
    \end{figure}

    \begin{figure}
        \centering
        \includegraphics[width=1\textwidth]{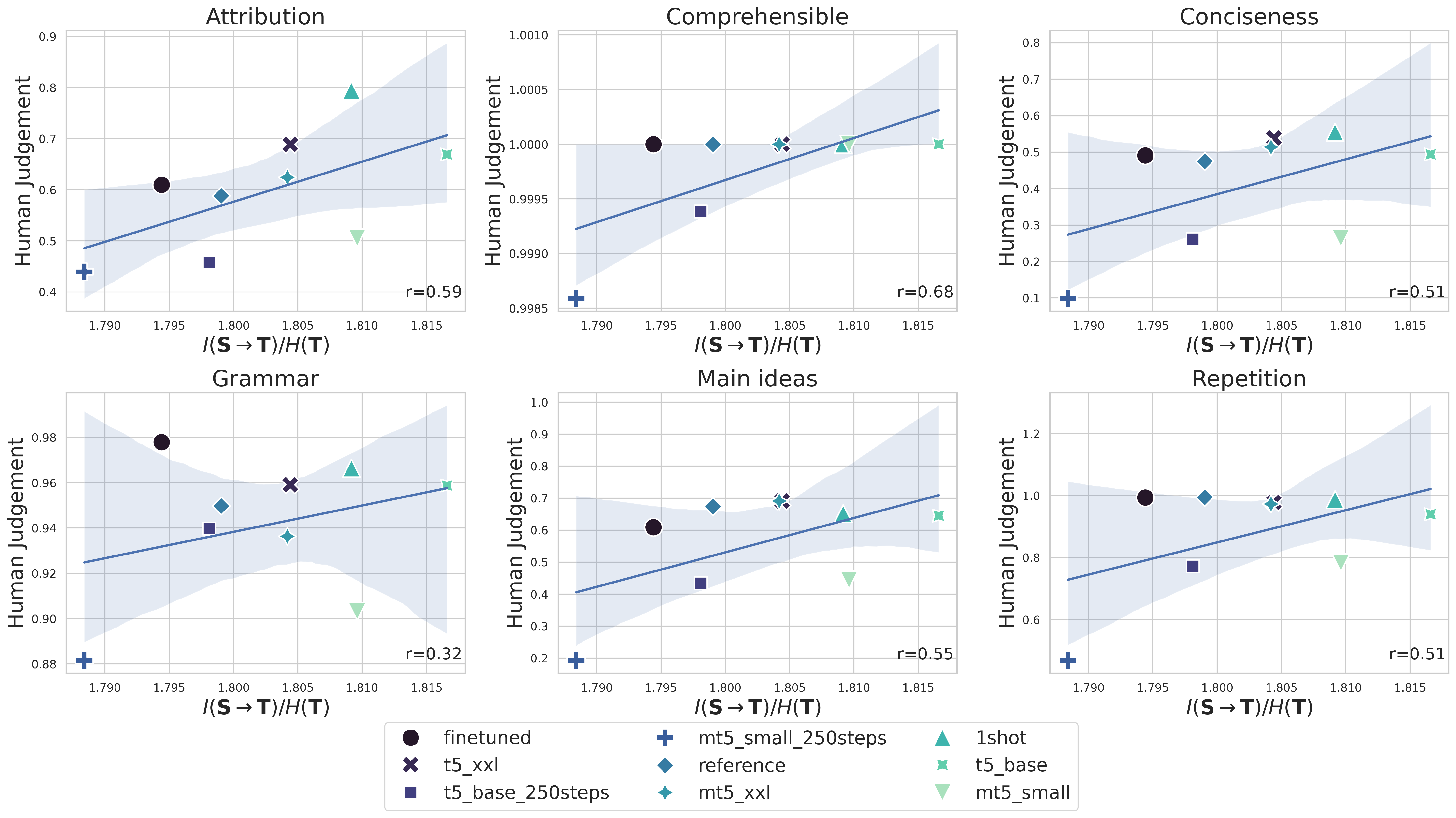}
        \caption{Correlation between the normalized mutual information and the answers to the SEAHORSE benchmark.}
        \label{fig:correlation_with_human_judgment}
    \end{figure}

    \subsection{SummEval dataset~\cite{fabbri2021summeval}}

   The SummEval dataset is well-suited for our task due to its inclusion of both summaries and corresponding source texts for identical documents. However, its limited size, comprising only 1700 samples, renders it insufficient for estimating mutual information.

   \subsection{Additional languages}
   \label{sec:additional_languages}

   We performed additional experiments in French, German and Spanish using multlingual embedder to evaluate the mutual informatio. We obtained mixed-results. While the overall trends are similar, the lack of good multilingual embedders certainly hinders the results we can hope to obtain. It is a clear limit of our work since our method is highly dependent on the existence of a viable embedder for the text distribution at hand.

  \begin{table}\centering
\caption{Correlation between MI and ROUGE, and Seahorse metrics and probability of success of the classifcation task, grouped by datasets for non-trivial decoding strategies. SH. stands for Seahorse metrics and CT. for classification tasks.}
\label{tab:correlation_table}
\resizebox{0.5\textwidth}{!}{\begin{tabular}{llrrrrr}
\toprule
 & Metric & $I(\textbf{S};\textbf{T})$ & Attribution & BERTScore & Main ideas & ROUGE-L \\
\midrule
\multirow[c]{6}{*}{SH.} & Attribution & {\cellcolor[HTML]{3D4E8A}} \color[HTML]{F1F1F1} 0.39 & {\cellcolor[HTML]{FDE725}} \color[HTML]{000000} 1.00 & {\cellcolor[HTML]{4EC36B}} \color[HTML]{000000} 0.78 & {\cellcolor[HTML]{6CCD5A}} \color[HTML]{000000} 0.82 & {\cellcolor[HTML]{3F4889}} \color[HTML]{F1F1F1} 0.38 \\
 & Comprehensible & {\cellcolor[HTML]{440154}} \color[HTML]{F1F1F1} 0.10 & {\cellcolor[HTML]{3F4889}} \color[HTML]{F1F1F1} 0.38 & {\cellcolor[HTML]{2E6E8E}} \color[HTML]{F1F1F1} 0.49 & {\cellcolor[HTML]{306A8E}} \color[HTML]{F1F1F1} 0.47 & {\cellcolor[HTML]{440154}} \color[HTML]{F1F1F1} 0.08 \\
 & Conciseness & {\cellcolor[HTML]{2C718E}} \color[HTML]{F1F1F1} 0.50 & {\cellcolor[HTML]{60CA60}} \color[HTML]{000000} 0.81 & {\cellcolor[HTML]{75D054}} \color[HTML]{000000} 0.83 & {\cellcolor[HTML]{ADDC30}} \color[HTML]{000000} 0.90 & {\cellcolor[HTML]{32658E}} \color[HTML]{F1F1F1} 0.46 \\
 & Grammar & {\cellcolor[HTML]{440154}} \color[HTML]{F1F1F1} 0.14 & {\cellcolor[HTML]{482374}} \color[HTML]{F1F1F1} 0.28 & {\cellcolor[HTML]{3D4E8A}} \color[HTML]{F1F1F1} 0.39 & {\cellcolor[HTML]{365C8D}} \color[HTML]{F1F1F1} 0.43 & {\cellcolor[HTML]{440154}} \color[HTML]{F1F1F1} 0.04 \\
 & Main ideas & {\cellcolor[HTML]{2A768E}} \color[HTML]{F1F1F1} 0.51 & {\cellcolor[HTML]{6CCD5A}} \color[HTML]{000000} 0.82 & {\cellcolor[HTML]{ADDC30}} \color[HTML]{000000} 0.90 & {\cellcolor[HTML]{FDE725}} \color[HTML]{000000} 1.00 & {\cellcolor[HTML]{2C718E}} \color[HTML]{F1F1F1} 0.50 \\
 & Repetition & {\cellcolor[HTML]{440154}} \color[HTML]{F1F1F1} -0.31 & {\cellcolor[HTML]{440154}} \color[HTML]{F1F1F1} -0.17 & {\cellcolor[HTML]{440154}} \color[HTML]{F1F1F1} -0.14 & {\cellcolor[HTML]{440154}} \color[HTML]{F1F1F1} -0.12 & {\cellcolor[HTML]{440154}} \color[HTML]{F1F1F1} -0.57 \\
\cline{1-7}
\multirow[c]{5}{*}{CT.} & Topic & {\cellcolor[HTML]{443983}} \color[HTML]{F1F1F1} 0.33 & {\cellcolor[HTML]{440154}} \color[HTML]{F1F1F1} 0.19 & {\cellcolor[HTML]{482374}} \color[HTML]{F1F1F1} 0.28 & {\cellcolor[HTML]{481D6F}} \color[HTML]{F1F1F1} 0.26 & {\cellcolor[HTML]{1FA088}} \color[HTML]{F1F1F1} 0.65 \\
 & Emotions & {\cellcolor[HTML]{440154}} \color[HTML]{F1F1F1} 0.10 & {\cellcolor[HTML]{440154}} \color[HTML]{F1F1F1} 0.06 & {\cellcolor[HTML]{440154}} \color[HTML]{F1F1F1} 0.09 & {\cellcolor[HTML]{440154}} \color[HTML]{F1F1F1} 0.08 & {\cellcolor[HTML]{3F4889}} \color[HTML]{F1F1F1} 0.37 \\
 & Sentiment Analysis & {\cellcolor[HTML]{440154}} \color[HTML]{F1F1F1} -0.04 & {\cellcolor[HTML]{440154}} \color[HTML]{F1F1F1} -0.09 & {\cellcolor[HTML]{440154}} \color[HTML]{F1F1F1} -0.10 & {\cellcolor[HTML]{440154}} \color[HTML]{F1F1F1} -0.15 & {\cellcolor[HTML]{440154}} \color[HTML]{F1F1F1} -0.00 \\
 & GPT Detector & {\cellcolor[HTML]{2A768E}} \color[HTML]{F1F1F1} 0.51 & {\cellcolor[HTML]{463480}} \color[HTML]{F1F1F1} 0.32 & {\cellcolor[HTML]{2E6E8E}} \color[HTML]{F1F1F1} 0.49 & {\cellcolor[HTML]{38588C}} \color[HTML]{F1F1F1} 0.42 & {\cellcolor[HTML]{31B57B}} \color[HTML]{F1F1F1} 0.72 \\
 & Policy & {\cellcolor[HTML]{26AD81}} \color[HTML]{F1F1F1} 0.69 & {\cellcolor[HTML]{1E9C89}} \color[HTML]{F1F1F1} 0.64 & {\cellcolor[HTML]{31B57B}} \color[HTML]{F1F1F1} 0.72 & {\cellcolor[HTML]{37B878}} \color[HTML]{F1F1F1} 0.74 & {\cellcolor[HTML]{24878E}} \color[HTML]{F1F1F1} 0.57 \\
\cline{1-7}
\multirow[c]{3}{*}{Emb.} & MPNET & {\cellcolor[HTML]{26AD81}} \color[HTML]{F1F1F1} 0.69 & {\cellcolor[HTML]{1E9C89}} \color[HTML]{F1F1F1} 0.64 & {\cellcolor[HTML]{3FBC73}} \color[HTML]{F1F1F1} 0.75 & {\cellcolor[HTML]{46C06F}} \color[HTML]{F1F1F1} 0.76 & {\cellcolor[HTML]{24878E}} \color[HTML]{F1F1F1} 0.57 \\
 & all-MiniLM & {\cellcolor[HTML]{1FA088}} \color[HTML]{F1F1F1} 0.65 & {\cellcolor[HTML]{1FA088}} \color[HTML]{F1F1F1} 0.65 & {\cellcolor[HTML]{37B878}} \color[HTML]{F1F1F1} 0.74 & {\cellcolor[HTML]{3FBC73}} \color[HTML]{F1F1F1} 0.75 & {\cellcolor[HTML]{297A8E}} \color[HTML]{F1F1F1} 0.53 \\
 & Paraphrase & {\cellcolor[HTML]{460B5E}} \color[HTML]{F1F1F1} 0.22 & {\cellcolor[HTML]{482374}} \color[HTML]{F1F1F1} 0.28 & {\cellcolor[HTML]{460B5E}} \color[HTML]{F1F1F1} 0.22 & {\cellcolor[HTML]{450457}} \color[HTML]{F1F1F1} 0.21 & {\cellcolor[HTML]{440154}} \color[HTML]{F1F1F1} -0.07 \\
\cline{1-7}
\multirow[c]{2}{*}{Common} & ROUGE-L & {\cellcolor[HTML]{277F8E}} \color[HTML]{F1F1F1} 0.54 & {\cellcolor[HTML]{3F4889}} \color[HTML]{F1F1F1} 0.38 & {\cellcolor[HTML]{2A768E}} \color[HTML]{F1F1F1} 0.51 & {\cellcolor[HTML]{2C718E}} \color[HTML]{F1F1F1} 0.50 & {\cellcolor[HTML]{FDE725}} \color[HTML]{000000} 1.00 \\
 & BERTScore & {\cellcolor[HTML]{2C718E}} \color[HTML]{F1F1F1} 0.50 & {\cellcolor[HTML]{4EC36B}} \color[HTML]{000000} 0.78 & {\cellcolor[HTML]{FDE725}} \color[HTML]{000000} 1.00 & {\cellcolor[HTML]{ADDC30}} \color[HTML]{000000} 0.90 & {\cellcolor[HTML]{2A768E}} \color[HTML]{F1F1F1} 0.51 \\
\cline{1-7}
\bottomrule
\end{tabular}}
\end{table}

  \subsection{Full results}
  \label{sec:table_radar}
  \begin{table}\centering
\caption{Correlation between MI and ROUGE, and Seahorse metrics and probability of success of the classifcation task, grouped by datasets for non-trivial decoding strategies. SH. stands for Seahorse metrics and CT. for classification tasks.}
\label{tab:correlation_table}
\resizebox{0.5\textwidth}{!}{\begin{tabular}{llrrrrrr}
\toprule
 & Metric & $I(\mathbf{S};\mathbf{T})$ & Attribution & Main ideas & \texttt{BARTScore} & \texttt{BERTScore} & \texttt{ROUGE-L} \\
\midrule
\multirow[c]{6}{*}{SH.} & Attribution & {\cellcolor[HTML]{26828E}} \color[HTML]{F1F1F1} 0.56 & {\cellcolor[HTML]{FDE725}} \color[HTML]{000000} 1.00 & {\cellcolor[HTML]{481B6D}} \color[HTML]{F1F1F1} 0.26 & {\cellcolor[HTML]{D2E21B}} \color[HTML]{000000} 0.95 & {\cellcolor[HTML]{3FBC73}} \color[HTML]{F1F1F1} 0.75 & {\cellcolor[HTML]{39568C}} \color[HTML]{F1F1F1} 0.42 \\
 & Comprehensible & {\cellcolor[HTML]{440154}} \color[HTML]{F1F1F1} 0.11 & {\cellcolor[HTML]{440154}} \color[HTML]{F1F1F1} 0.07 & {\cellcolor[HTML]{375A8C}} \color[HTML]{F1F1F1} 0.42 & {\cellcolor[HTML]{440154}} \color[HTML]{F1F1F1} 0.10 & {\cellcolor[HTML]{440154}} \color[HTML]{F1F1F1} 0.02 & {\cellcolor[HTML]{440154}} \color[HTML]{F1F1F1} -0.37 \\
 & Conciseness & {\cellcolor[HTML]{5EC962}} \color[HTML]{000000} 0.80 & {\cellcolor[HTML]{20A486}} \color[HTML]{F1F1F1} 0.67 & {\cellcolor[HTML]{65CB5E}} \color[HTML]{000000} 0.81 & {\cellcolor[HTML]{1FA188}} \color[HTML]{F1F1F1} 0.66 & {\cellcolor[HTML]{20A486}} \color[HTML]{F1F1F1} 0.67 & {\cellcolor[HTML]{481467}} \color[HTML]{F1F1F1} 0.24 \\
 & Grammar & {\cellcolor[HTML]{440154}} \color[HTML]{F1F1F1} -0.24 & {\cellcolor[HTML]{440154}} \color[HTML]{F1F1F1} -0.35 & {\cellcolor[HTML]{440154}} \color[HTML]{F1F1F1} 0.16 & {\cellcolor[HTML]{440154}} \color[HTML]{F1F1F1} -0.34 & {\cellcolor[HTML]{440154}} \color[HTML]{F1F1F1} -0.30 & {\cellcolor[HTML]{440154}} \color[HTML]{F1F1F1} -0.50 \\
 & Main ideas & {\cellcolor[HTML]{4AC16D}} \color[HTML]{000000} 0.77 & {\cellcolor[HTML]{481B6D}} \color[HTML]{F1F1F1} 0.26 & {\cellcolor[HTML]{FDE725}} \color[HTML]{000000} 1.00 & {\cellcolor[HTML]{471063}} \color[HTML]{F1F1F1} 0.23 & {\cellcolor[HTML]{33638D}} \color[HTML]{F1F1F1} 0.45 & {\cellcolor[HTML]{482475}} \color[HTML]{F1F1F1} 0.28 \\
 & Repetition & {\cellcolor[HTML]{440154}} \color[HTML]{F1F1F1} 0.01 & {\cellcolor[HTML]{440154}} \color[HTML]{F1F1F1} -0.23 & {\cellcolor[HTML]{3D4D8A}} \color[HTML]{F1F1F1} 0.39 & {\cellcolor[HTML]{440154}} \color[HTML]{F1F1F1} -0.23 & {\cellcolor[HTML]{440154}} \color[HTML]{F1F1F1} -0.08 & {\cellcolor[HTML]{440154}} \color[HTML]{F1F1F1} -0.34 \\
\cline{1-8}
\multirow[c]{4}{*}{CT.} & Sentiment analysis & {\cellcolor[HTML]{1FA188}} \color[HTML]{F1F1F1} 0.65 & {\cellcolor[HTML]{22A785}} \color[HTML]{F1F1F1} 0.68 & {\cellcolor[HTML]{433E85}} \color[HTML]{F1F1F1} 0.34 & {\cellcolor[HTML]{22A785}} \color[HTML]{F1F1F1} 0.68 & {\cellcolor[HTML]{27AD81}} \color[HTML]{F1F1F1} 0.70 & {\cellcolor[HTML]{277F8E}} \color[HTML]{F1F1F1} 0.54 \\
 & GPT detector & {\cellcolor[HTML]{32B67A}} \color[HTML]{F1F1F1} 0.73 & {\cellcolor[HTML]{75D054}} \color[HTML]{000000} 0.83 & {\cellcolor[HTML]{3C4F8A}} \color[HTML]{F1F1F1} 0.39 & {\cellcolor[HTML]{8ED645}} \color[HTML]{000000} 0.86 & {\cellcolor[HTML]{A2DA37}} \color[HTML]{000000} 0.89 & {\cellcolor[HTML]{1F9E89}} \color[HTML]{F1F1F1} 0.65 \\
 & Policy classification & {\cellcolor[HTML]{70CF57}} \color[HTML]{000000} 0.83 & {\cellcolor[HTML]{3C508B}} \color[HTML]{F1F1F1} 0.40 & {\cellcolor[HTML]{25AC82}} \color[HTML]{F1F1F1} 0.69 & {\cellcolor[HTML]{31668E}} \color[HTML]{F1F1F1} 0.46 & {\cellcolor[HTML]{48C16E}} \color[HTML]{F1F1F1} 0.77 & {\cellcolor[HTML]{2DB27D}} \color[HTML]{F1F1F1} 0.71 \\
 & Emotion classification & {\cellcolor[HTML]{2FB47C}} \color[HTML]{F1F1F1} 0.72 & {\cellcolor[HTML]{23A983}} \color[HTML]{F1F1F1} 0.68 & {\cellcolor[HTML]{3B518B}} \color[HTML]{F1F1F1} 0.40 & {\cellcolor[HTML]{2EB37C}} \color[HTML]{F1F1F1} 0.72 & {\cellcolor[HTML]{40BD72}} \color[HTML]{F1F1F1} 0.75 & {\cellcolor[HTML]{238A8D}} \color[HTML]{F1F1F1} 0.58 \\
\cline{1-8}
Emb. & Emb. Paraphrase & {\cellcolor[HTML]{58C765}} \color[HTML]{000000} 0.79 & {\cellcolor[HTML]{42BE71}} \color[HTML]{F1F1F1} 0.76 & {\cellcolor[HTML]{218F8D}} \color[HTML]{F1F1F1} 0.60 & {\cellcolor[HTML]{37B878}} \color[HTML]{F1F1F1} 0.74 & {\cellcolor[HTML]{26AD81}} \color[HTML]{F1F1F1} 0.69 & {\cellcolor[HTML]{472A7A}} \color[HTML]{F1F1F1} 0.29 \\
\cline{1-8}
\multirow[c]{3}{*}{Common} & \texttt{ROUGE-L} & {\cellcolor[HTML]{26828E}} \color[HTML]{F1F1F1} 0.55 & {\cellcolor[HTML]{39568C}} \color[HTML]{F1F1F1} 0.42 & {\cellcolor[HTML]{482475}} \color[HTML]{F1F1F1} 0.28 & {\cellcolor[HTML]{34618D}} \color[HTML]{F1F1F1} 0.45 & {\cellcolor[HTML]{26AD81}} \color[HTML]{F1F1F1} 0.70 & {\cellcolor[HTML]{FDE725}} \color[HTML]{000000} 1.00 \\
 & \texttt{BERTScore} & {\cellcolor[HTML]{58C765}} \color[HTML]{000000} 0.79 & {\cellcolor[HTML]{3FBC73}} \color[HTML]{F1F1F1} 0.75 & {\cellcolor[HTML]{33638D}} \color[HTML]{F1F1F1} 0.45 & {\cellcolor[HTML]{6CCD5A}} \color[HTML]{000000} 0.82 & {\cellcolor[HTML]{FDE725}} \color[HTML]{000000} 1.00 & {\cellcolor[HTML]{26AD81}} \color[HTML]{F1F1F1} 0.70 \\
 & \texttt{BARTScore} & {\cellcolor[HTML]{23888E}} \color[HTML]{F1F1F1} 0.57 & {\cellcolor[HTML]{D2E21B}} \color[HTML]{000000} 0.95 & {\cellcolor[HTML]{471063}} \color[HTML]{F1F1F1} 0.23 & {\cellcolor[HTML]{FDE725}} \color[HTML]{000000} 1.00 & {\cellcolor[HTML]{6CCD5A}} \color[HTML]{000000} 0.82 & {\cellcolor[HTML]{34618D}} \color[HTML]{F1F1F1} 0.45 \\
\cline{1-8}
\bottomrule
\end{tabular}}
\end{table}

    \section{Deciphering Summarizer Hierarchy}
    \label{sec:summarizers_hierarchy}
    
    We proposed to evaluate the mutual information $ I(\mathbf{T};\mathbf{S}) $, where $\mathbf{S} \follows p_\theta (\mathbf{s}|\mathbf{t})$ being a summarizer --in our case, a finetuned language model-- and $\mathbf{T}$ is the random variable of source texts. If we have two summarizers $p_\theta (\mathbf{s}|\mathbf{t})$ and $ q_\phi  (\mathbf{s}|\mathbf{t})$, we can evaluate the mutual information $I(\mathbf{S}_p \rightarrow \mathbf{S}_q ) $, where $\mathbf{S}_p \follows p_\theta (\mathbf{s}|\mathbf{t})$ and $\mathbf{S}_q \follows q_\phi (\mathbf{s}|\mathbf{t}) $. The mutual information here indicates how much information about $\mathbf{S}_q$ conveys about $\mathbf{S}_p$ and vice-versa. Interestingly, this observation enables us to build a hierarchy of summarizers. Some summarizers produce very informative summaries that can be used to predict the ones from other models while being so informative that other summaries cannot provide enough information to build them. We build the directed graph of the predictive power of the summaries of each model on the summaries of other models. A model's average outgoing mutual information is the average mutual information between this model's summaries and other models' summaries. A model's average incoming mutual information is the average mutual information between its summaries and those of other models.

    \noindent\textbf{OOD models.}    Underperforming models, which were trained on disparate data distributions such as Arxiv or medical summarization, generally display low mutual information with other models and prove challenging to predict from conventional specialized systems (see F\autoref{fig:model_hierarchy}). This outcome is unsurprising, given that these models exhibit significantly divergent behavior compared to others. Consequently, their outputs offer minimal insight into the outputs of other models.

    \noindent\textbf{Strong models.}     Robust models like distilBart and Bart demonstrate high informativeness regarding other models, while also posing challenges for prediction (refer to \autoref{fig:model_hierarchy} and \autoref{fig:model_comparison}). This outcome is anticipated, given that robust models can encapsulate significantly more information within their summaries compared to other models. As a result, their summaries prove valuable for predicting the outputs of weaker summarizers. However, these summaries are also considerably challenging to predict from the perspectives of those weaker models.

    \begin{figure}[H]
        \centering
        \includegraphics[width=.95\textwidth]{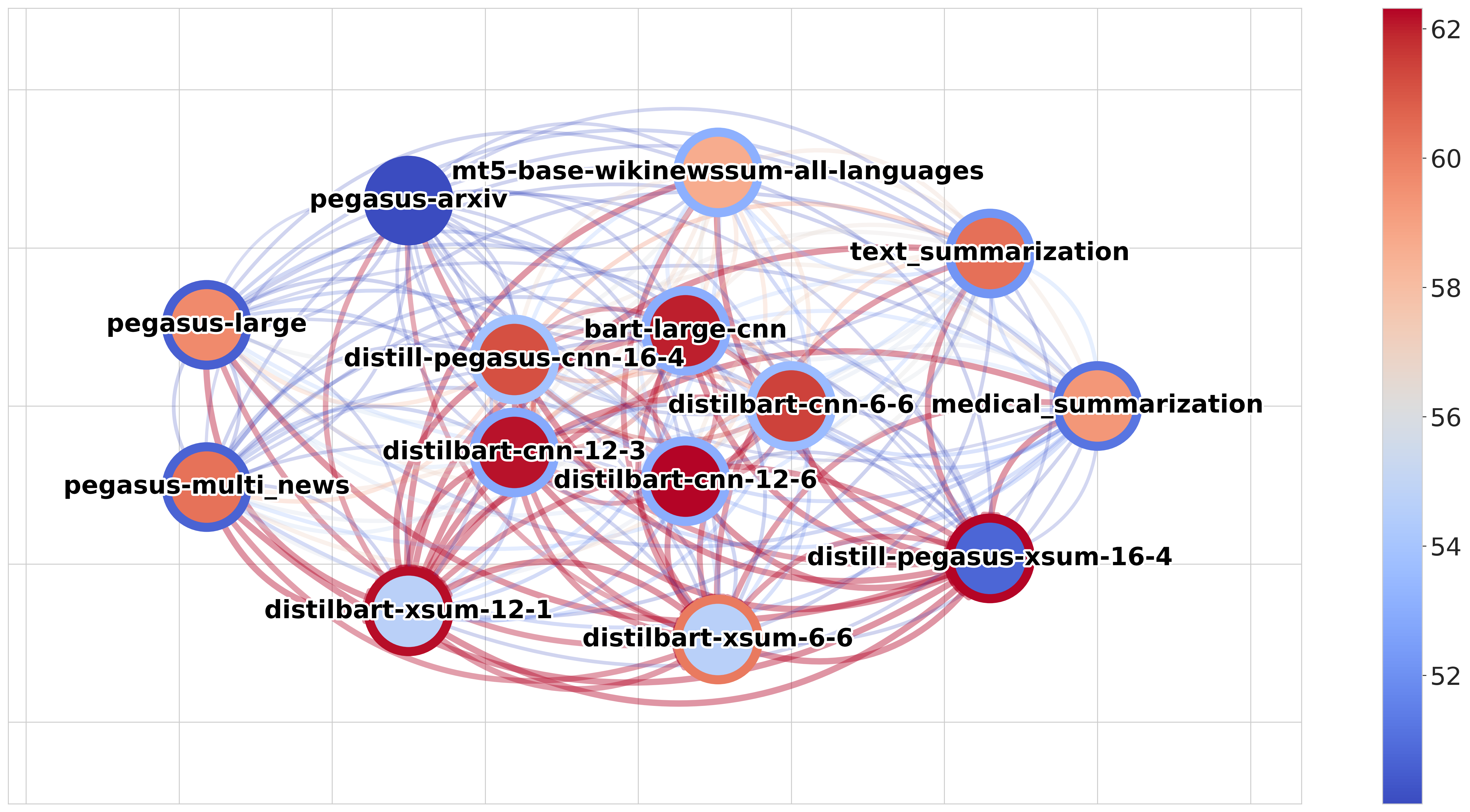}
        \caption{The predictive power of each model's summaries on the summaries of other models is depicted in the visualization. The central color denotes the average predictive power of that summarizer regarding the others, while the border color indicates the average predictive power of the other summarizers concerning that summarizer. A red center and blue border signify high informativeness, indicating a summarizer that is highly informative and difficult to predict. Conversely, a node with a blue center and red border implies low informativeness about the other summarizers but easy predictability by them.}
        \label{fig:model_hierarchy}
    \end{figure}

    \section{Ablations}

    \subsection{Decoding size}
    \label{sec:ap_decoding_size}
    \begin{figure}
        \centering
        \includegraphics[width=.95\textwidth]{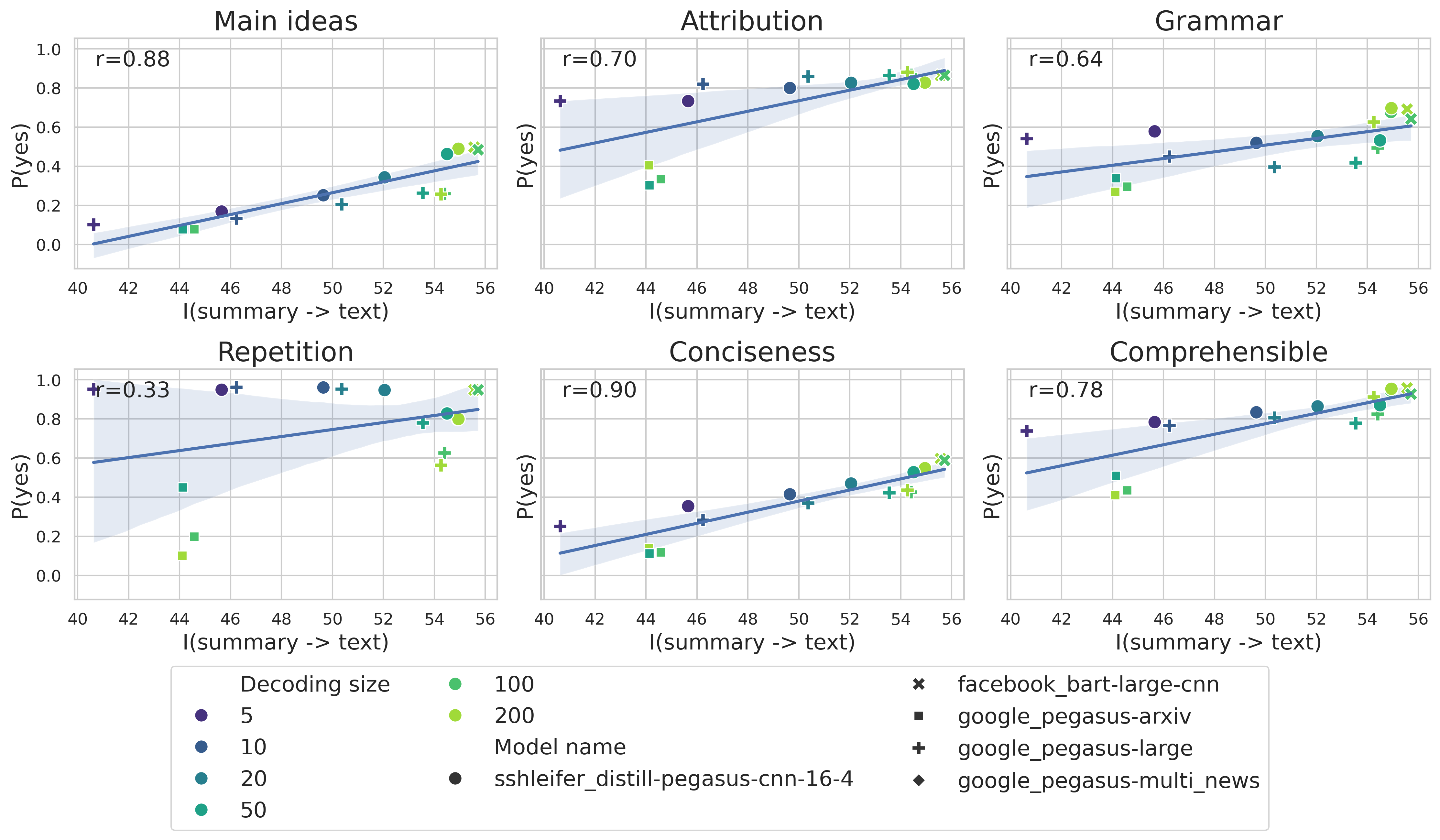}
        \caption{Correlation with the SEAHORSE metrics on the CNN DailyMail dataset is measured by $\Pr(\text{Yes})$, which represents the average probability across the dataset that the SEAHORSE model predicts the answer "Yes" to the corresponding question.}
        \label{fig:corr_sh_metrics_ablation}
    \end{figure}

    \noindent\textbf{Impact of the length of the summary.} The longer a summary, the more likely a downstream classifier is to produce the same output on the source text and on the summary. However, this trend is not always verified for weakers or OOD models. In \autoref{fig:classification_decoding_size} and \autoref{fig:corr_sh_metrics_ablation}, we can observe that the Pegasus model finetuned on arxiv papers tends to be less informative even when generating extended summaries when applied to the CNN-Dailymail dataset. This shows that the mutual information captures more than just the length of the summary but also its actual informativity.

    \begin{figure}
        \centering
        \includegraphics[width=.95\textwidth, trim=1cm 1cm 0cm 0]{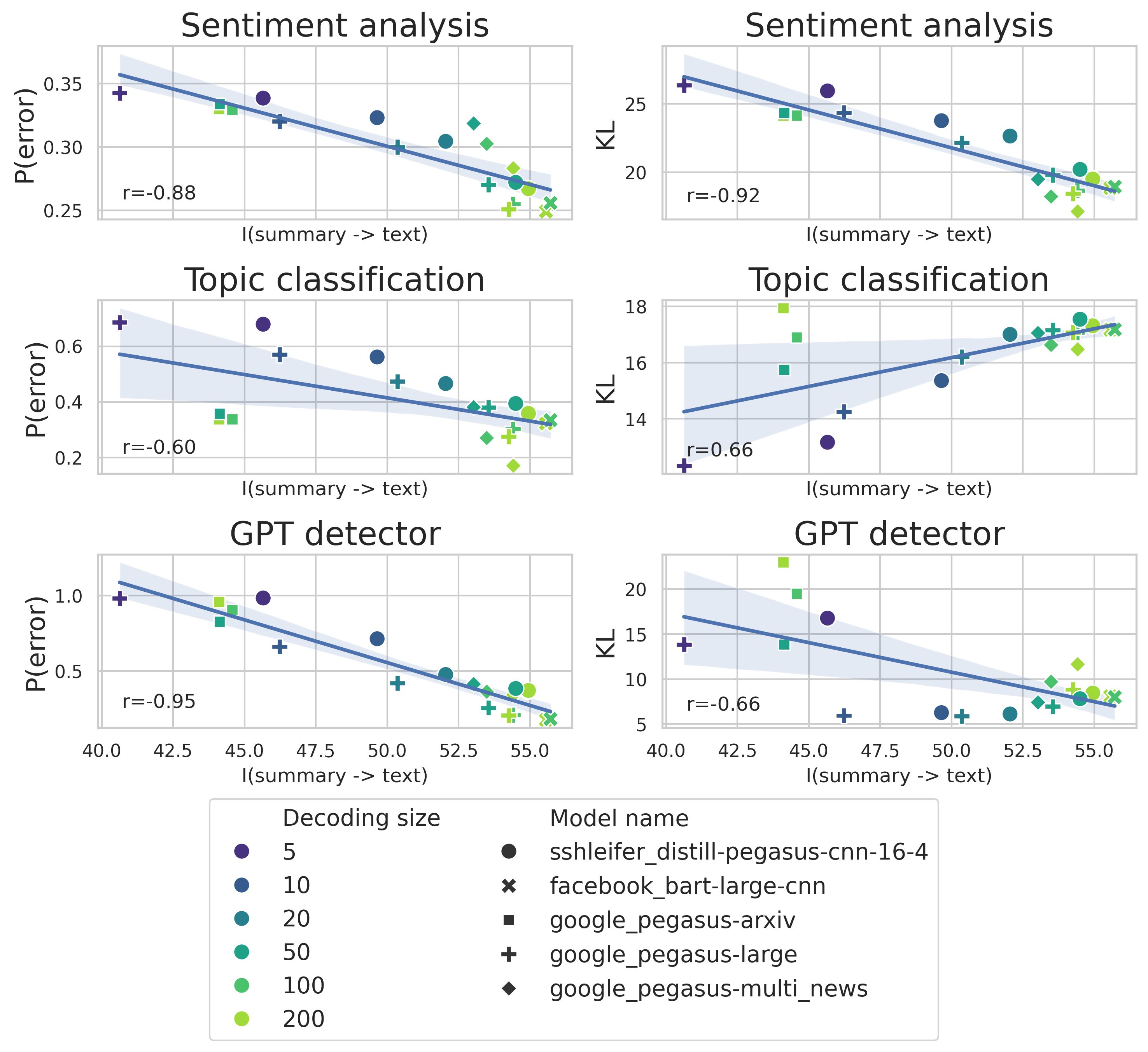}
        \vspace{2mm}
        \caption{As one would expect, the performance of the classification tasks on the summaries increases with the decoding size as it allows the model to pack more information into the summary.}
        \label{fig:classification_decoding_size}
    \end{figure}

    \section{Negative Results}

    \subsection{Pointwise mutual information}

    While the mutual information gives good insights about a summarizer, the point-wise mutual information, computed for each pair of source texts and summaries did not result in interesting correlations with the downstream tasks. Previous work have shown that it was a sound metric when the generative model is used to compute the mutual information~\cite{bugliarello2020its, ethayarajh2022understanding}, however in our scenario we fit an ad-hoc mutual information estimator.

\end{document}